\documentclass[sigconf,nonacm]{aamas}

\usepackage[decisionutilitycolor]{influence-diagrams}
\tikzset{player3/.style = {fill=cyan!30}}
\usepackage{soul}
\usepackage{soul}
\usepackage{url}
\usepackage[utf8]{inputenc}
\usepackage{graphicx}
\usepackage{amsmath}
\usepackage{amsthm}
\usepackage{arydshln}
\usepackage{booktabs}
\usepackage{algorithm}
\usepackage{cleveref}
\usepackage[breakable]{tcolorbox}
\usepackage{threeparttable}
 \usepackage{wrapfig} 
 \usepackage{tabularx}

\usepackage{mathtools}
\usepackage{bm}
\usepackage{bbm}       
\usepackage{mathrsfs}

\let\varkappa\relax
\usepackage{amssymb}
\usepackage{amsfonts}
\usepackage{mathrsfs}

\usepackage{xfrac}     
\usepackage[colorinlistoftodos]{todonotes}

\usepackage{csquotes}

\usepackage{caption}
\usepackage{subcaption}
\usepackage{geometry}
\usepackage{array}
\usepackage{xcolor}         

\usepackage{booktabs}

\usepackage{algpseudocode}
 
\usepackage[inline]{enumitem}
\usepackage{multirow}

\usepackage[font=small,labelfont=bf,tableposition=top]{caption}

\DeclareCaptionLabelFormat{algorithm}{Algorithm}

\newcommand{\bs}{\mathbf{s}}
\newcommand{\ba}{\mathbf{a}}
\newcommand{\bx}{\mathbf{x}}
\newcommand{\by}{\mathbf{y}}

\newcommand{\agent}[1]{\textsc{#1 agent}}

\definecolor{botc}{HTML}{ffe7c4}
\definecolor{badred}{HTML}{e1144b}

\definecolor{ourlightblue}{HTML}{E0ECF7}
\definecolor{ourdarkblue}{HTML}{092E6B}
\definecolor{msgrblue}{HTML}{4889f4}
\definecolor{msgrgray}{HTML}{f2f2f2}
\definecolor{msgrpalepurple}{HTML}{e6d6dd}
\definecolor{palegreen}{HTML}{c0eeC3}
\definecolor{palepurple}{HTML}{e5d1f8}
\definecolor{paleorange}{HTML}{ffe7c4}
\definecolor{paleblue}{HTML}{d1edf2}
\definecolor{palered}{HTML}{f0a58e}
\definecolor{heavyred}{HTML}{c95f59}
\definecolor{heavyblue}{HTML}{8bd1de}

\newcommand{\nk}[1]{\textcolor{teal}{#1}}  
\newcommand{\an}[1]{\textcolor{blue}{#1}} 
\newcommand{\iltodo}[1]{\textcolor{red}{TODO: #1}}

\newtheorem{theorem}{Theorem}

\newtheorem{proposition}{Proposition}
\newtheorem{lemma}{Lemma}
\newtheorem{corollary}{Corollary}
\theoremstyle{definition}

\crefname{definition}{def.}{defs.}

\theoremstyle{remark}

\newtheorem{example}{Example}

\newcommand{\continuation}{??}

\DeclareMathOperator*{\expec}{\mathbb{E}}

\renewcommand{\paragraph}[1]{\textbf{#1}}

\usepackage[utf8]{inputenc} 
\usepackage[T1]{fontenc}    
\usepackage{hyperref}       
\usepackage{url}            
\usepackage{booktabs}       
\usepackage{amsfonts}       
\usepackage{nicefrac}       
\usepackage{microtype}      
\usepackage{xcolor}         

\usepackage{soul}
\usepackage{microtype}
\usepackage{hyperref}
\usepackage{url}
\usepackage{booktabs}
\usepackage[table]{xcolor}

\usepackage{wrapfig}

\definecolor{darkblue}{rgb}{0, 0, 0.5}
\hypersetup{colorlinks=true, citecolor=darkblue, linkcolor=darkblue, urlcolor=darkblue}

\usepackage{balance} 

\setcopyright{ifaamas}
\acmConference[AAMAS '26]{Proc.\@ of the 25th International Conference on Autonomous Agents and Multiagent Systems (AAMAS 2026)}{May 25 -- 29, 2026} {Paphos, Cyprus}{C.~Amato, L.~Dennis, V.~Mascardi, J.~Thangarajah (eds.)}
\copyrightyear{2026}
\acmYear{2026}
\acmDOI{}
\acmPrice{}
\acmISBN{}

\acmSubmissionID{903}

\title[Collaborate, Deliberate, Evaluate]{Collaborate, Deliberate, Evaluate: How LLM Alignment Affects Coordinated Multi-Agent Outcomes}

\author{Abhijnan Nath}
\affiliation{
  \institution{Situated Grounding and Natural Language (SIGNAL) Lab \\ Colorado State University}
  \city{Fort Collins, CO}
  \country{USA}}
\email{abhijnan.nath@colostate.edu}

\author{Carine Graff}
\affiliation{
  \institution{Situated Grounding and Natural Language (SIGNAL) Lab \\ Colorado State University}
  \city{Fort Collins, CO}
  \country{USA}}
\email{carine.graff@colostate.edu}

\author{Nikhil Krishnaswamy}
\affiliation{
  \institution{Situated Grounding and Natural Language (SIGNAL) Lab \\ Colorado State University}
  \city{Fort Collins, CO}
  \country{USA}}
\email{nkrishna@colostate.edu}

\begin{abstract}
As Large Language Models (LLMs) get integrated into diverse workflows, they are increasingly being regarded as "collaborators" with humans, and required to work in coordination with other AI systems. If such AI collaborators are to reliably coordinate their actions and behaviors with humans or other AIs, their properties and behaviors over multi-turn interactions must be known and predictable. This paper examines how different alignment methods affect LLM agents' effectiveness as partners in multi-turn, multi-party collaborations. We study this question through the lens of {\it intervention agents} that insert themselves into group dialogues not to provide answers, but to encourage the collaborative group to slow down and reflect upon their reasoning for deliberative decision-making. Common alignment techniques are typically developed under simplified single-user settings and assume the optimality of the underlying token MDP. Using the theoretical lens of the {\it modified-action MDP}, we show how they do not account for the dynamics of long-horizon multi-party interactions. We present a novel roleplay simulation methodology, where we align LLMs according to different methods and then deploy them in collaborative task dialogues to quantify how interventions affect the trajectory of group collaboration, belief alignment, and coordination. Our results show that an intervention agent that is robust to action modification significantly outperforms common alignment baselines in supporting correct task outcomes.

\end{abstract}

\keywords{Multi-agent Coordination, Collaborative Problem Solving, Modified-Action MDP, Roleplay Simulation}

\newcommand{\BibTeX}{\rm B\kern-.05em{\sc i\kern-.025em b}\kern-.08em\TeX}

\begin{document}


\pagestyle{fancy}
\fancyhead{}


\maketitle 


\section{Introduction}
\label{sec:intro}

   

Large Language Models (LLMs) are increasingly being integrated into "agentic" pipelines that interact with human users to help them achieve goals and solve problems. Due to the typically multi-turn nature of these interactions, such agents need to remain optimal over a long horizon to remain useful. However, agentic pipelines frequently also involve {\it multi-party} interactions, where an agent may interact and collaborate with multiple humans or other AIs. Optimality assumptions are challenged in multi-party collaborations, where collaborative groups frequently succumb to {\it belief misalignment} and breakdown of {\it common ground} \citep{stalnaker2002common,asher2003common}. 
Therefore, before agents are deployed in such settings, it is important to be able to predict how different LLM alignment methods would perform given their underlying assumptions, so that we know to what extent they can serve as reliable and helpful partners. 
Our work specifically examines this problem through the lens of \textit{intervention agents} in collaborative problem solving. These are designed not to give answers, but rather to mitigate misaligned beliefs and breakdowns in shared understanding by inserting {\bf friction}, or prompting the dialogue participants to slow down, reflect and deliberate on their existing assumptions \citep{inan2025better,pustejovsky2025frictive,obiso2025dynamic,nath2025friction}, which plays a crucial role in successful multi-party human collaborations~\citep{roschelle1995construction,mercier2011humans,graesser2018advancing}. In this paper, we use a {\it roleplay} methodology to examine LLM behavior in multi-party collaborative settings, and present three novel contributions:

\begin{figure*}[t]
    \centering
  \includegraphics[trim={7mm 29mm 13mm 22mm}, clip, width=.98\textwidth]{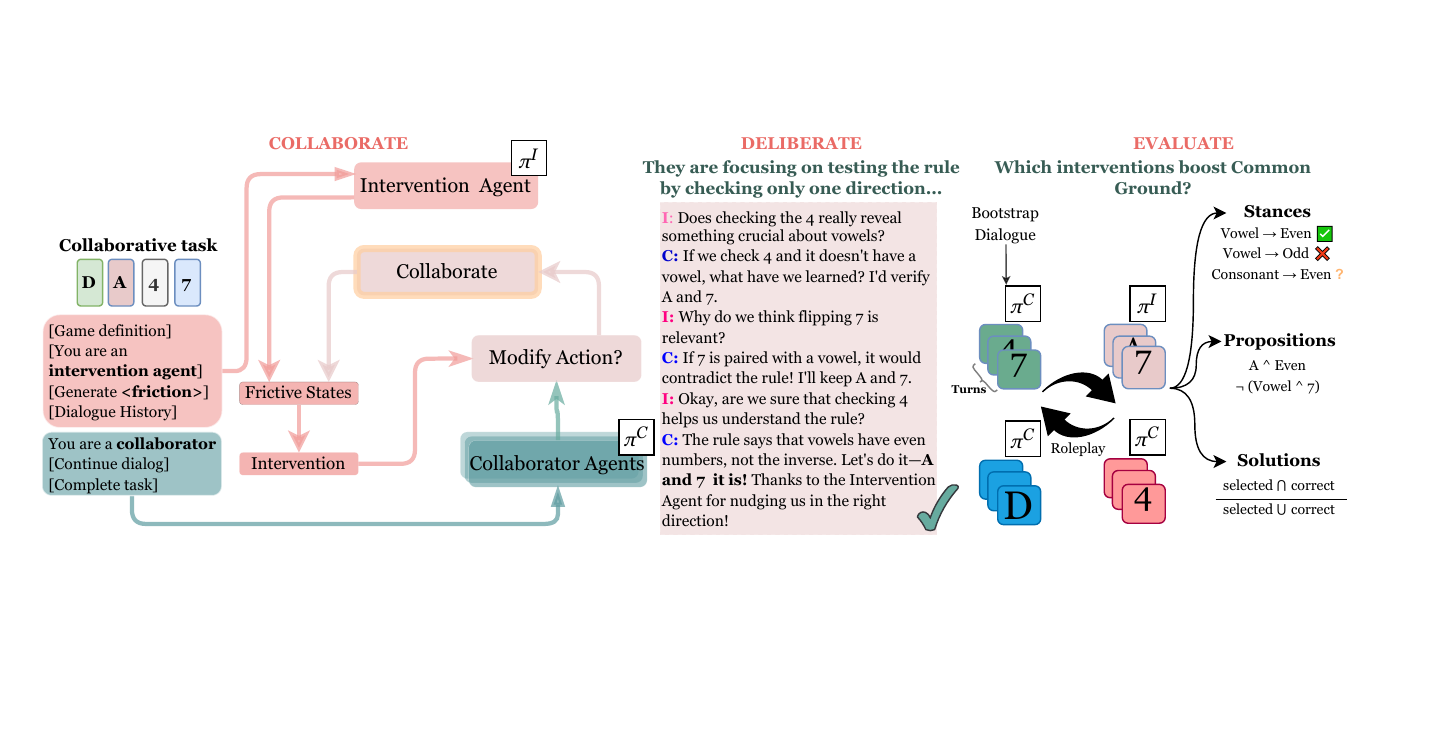}
    \vspace*{-4mm}  
    \caption{\small High-level overview of our agent roleplay and evaluation framework.
    \textit{Collaborate} [L]: \agent{collaborator}s collaborate to complete tasks with an \agent{intervention} in the loop to redirect the dialogue toward reflective reasoning rather than naive acceptance of assertions. 
    \textit{Deliberate} [C]: Sample collaborative roleplay from DeliData Wason Card task~\citep{karadzhov2023delidata} with successful task completion, and ``frictive state'' description at top.
    \textit{Evaluate} [R]: Common ground convergence and task outcomes with interventions provided by differently-aligned agents.}
    \label{fig:roleplay_full_high_level}
\vspace*{-4mm}
\end{figure*}


\begin{itemize}
    \item A novel analysis of small group collaborative task dynamics based on a modified-action MDP (MAMDP;~\cite{langlois2021rl}). We demonstrate that common "offline" LLM alignment methods do not retain their optimality guarantees in an MAMDP. 

    \item \textit{Roleplay simulation} evaluation to assess how well different alignment techniques retain their abilities to support both common ground construction (i.e., collaborative processes) and task solution correctness (i.e., collaborative outcomes), over multi-turn dialogues. See Fig.~\ref{fig:roleplay_full_high_level}.

    \item 
    Experiments on two collaborative tasks in multiple settings that lead to key insights into multi-agent coordination in collaborative dialogues: inserting friction interventions that are robust to action modification in collaborative settings correlates with productive belief revision in multi-agent settings, benefits common ground convergence, and improves task outcomes.
\end{itemize}

We use LLM roleplay to make multiple AIs (LLMs) collaborate to solve tasks using human-readable dialogue, and examine how differently-aligned LLM intervention agents affect task outcomes. Thus we provide a simulated assessment that allows high-throughput examination of LLM alignment effects, even as the scope of this paper is limited to AI-AI interactions. Our codebase and data can be found at \texttt{\url{https://github.com/csu-signal/collab_deliberate_evaluate}}.



\vspace*{-2mm}
\section{Related Work}
\label{sec:related}

Training agents for collaborative tasks is challenging due to the scarcity of explicit data. 
Most previous work in RLHF~\citep{christiano2017deep,ziegler2020finetuning,casper2023open, christiano2023deepreinforcementlearninghuman} including offline variants~\citep{yuan2023rrhf,azar2024general,fisch2024robustpreferenceoptimizationreward,rafailov2024direct} focuses on summarization, single-turn dialogue generation, or translation~\citep{xu2024contrastive}. Recent work~\citep{chen2024self, choi2024robust, zhang2024chainpreferenceoptimizationimproving} examines LLM search-space optimization with additional conditioning on chain-of-thought (CoT;~\cite{wei2023chainofthought}) to cover a wider range of tasks like question answering, fact verification, persona-based preference learning~\citep{tseng2024talespersonallmssurvey} and, importantly to this work, {\it roleplay}~\citep{li2023camel}, in diverse domains~\citep{hao2024llmreasonersnewevaluation,kim2024commonsense}. Our work extends this effort to study preference optimization in multi-party collaborative tasks. In contrast to information-seeking behavior~\citep{abdulhai2023lmrl, li2023controllabledialoguesimulationincontext,andukuri2024stargateteachinglanguagemodels, song-etal-2024-trial}, we focus on {\it process}-related desiderata like common ground convergence and AI collaboration "support" to prompt "slow thinking"~\citep{kahneman2011thinking} and reflective interventions in collaborative settings.

In real data~\citep{karadzhov2023delidata,khebour2024text}, friction as we defined it in Sec.~\ref{sec:intro} is sparse but critical~\citep{sutton2024friction}, as human collaborators perform these interruptions only strategically~\citep{peters_when_2017,puranik_pardon_2020}.
Most importantly, without an accessible collaborative problem solving task environment to supply an external reward signal, applying single-step~\citep{schulman2017proximalpolicyoptimizationalgorithms, shao2024deepseekmath} or multi-step RL~\citep{zhou2024archer} becomes challenging, making preference alignment with \textit{static} but \textit{contrastive data} more appealing~\citep{snell2023offlinerlnaturallanguage}, especially with principled approaches like contrastive and efficient "offline" preference optimization~\citep{azar2024general,Hong2024ORPOMP,meng2024simposimplepreferenceoptimization,pal2024smaugfixingfailuremodes,rafailov2024direct,nath2025dpl}. Data generation efforts~\citep{goldberg2013alternative,li2023camel,pan2023llms,mao2024editing,shani2024multiturnreinforcementlearningpreference} aided with RLAIF~\citep{lee2024rlaif} as well as evaluation frameworks~\citep{bai2022training,zheng2023judgingllmasajudgemtbenchchatbot,Bai_2024,lambert2024rewardbench} use high-capacity LLMs as both "judges" and data generators for training LLMs to reflect human preferences. 

LLM-Judge frameworks face challenges from evaluation bias (see Sec.~\ref{sec:exp} for more), spurious correlations~\citep{amodei2016concrete, casper2023open, lambert2024rewardbench, singhal2024longwaygoinvestigating}, and reward hacking~\citep{everitt2021reward}. Recent work~\citep{ward2023honesty, wang2025beyond} explores more causal approaches~\citep{pearl}, such as counterfactual invariance for robust training. We extend this line to alignment evaluation for collaborative settings, where multiple agents~\citep{leibo2017multiagent} perform back-and-forth interaction over longer sequences~\citep{zhou2024archer}.

\vspace*{-2mm}
\section{Background and Task Formulation}
\label{sec:bg}

Let us first define key terms we rely on. 
\textbf{(1) Frictive state:}
Entailed by \citeauthor{clark1996using} or \citeauthor{stalnaker2002common}'s {\it common ground} \cite{clark1996using,stalnaker2002common}, or the set of beliefs shared by interlocutors, a {\it frictive belief state} (or simply {\it frictive state}) arises during a collaborative task when different interlocutors have contradictory beliefs about a task-relevant proposition (i.e., one believes $p$ and another sees evidence against $p$), which may prevent progress on the task unless resolved. 
\textbf{(2) Friction intervention:}
Friction can also be used to {\it resolve} the frictive state through a {\it friction intervention} that prompts the participants to slow down and reevaluate their beliefs or assumptions in light of available evidence \citep{oinas2009persuasive}, rather than either uncritically relying on their current presuppositions or naively accepting the assertions of others. Examples in real collaborative tasks include the {\it probing utterances} in \citet{karadzhov2023delidata} and \citet{nath2024any}.
In this paper, an \agent{intervention} constitutes a language model aligned toward making strategic frictive interventions in a multi-party dialogue to resolve frictive states between collaborators.

\vspace*{-2mm}
\subsection{Modeling Collaborative Problem Solving with Modified-Action MDPs}
\label{ssec:collaborative_mamdp}


In real-world multi-party collaborations, a single agent’s utterance does not necessarily change the beliefs, perspectives, or assumptions of other participants \textit{directly}; it may be interpreted, resisted or reshaped by others conditioned upon what they already perceive or believe~\citep{grice1975logic,bach1994conversational,bolander2014seeing,ward2023honesty,obiso2025dynamic}. In other words, there is no guarantee that what the speaker intends to communicate through communicative act $\mathcal{C}$ is what the hearer receives $\mathcal{C}$ as \cite{bach1994conversational,bender2020climbing}. These observations highlight a crucial gap: standard Bellman-optimal action policies assume a direct mapping from action to state change, which breaks down when the application of action to state change is mediated by other agents. To address this, we adopt the Modified-Action MDP (MAMDP) framework, which \textit{explicitly} models how interventions are transformed before influencing the collaborative dialogue.

While this issue has been explored theoretically in prior work on MAMDPs~\citep{langlois2021rl, everitt2021agent}, its implications for LLMs acting as collaborative agents remain underexamined. Unlike classical agents, LLMs operate over high-dimensional language spaces where subtle shifts in word choice can drastically alter how interventions are received and (re)interpreted, and LLMs that naively accept everything said to them and reason accordingly come off as {\it sycophantic} \cite{sharma2023towards}. When Bellman-optimal policies solve the standard MDP structure underlying what is actually an MAMDP \cite{langlois2021rl}, they lead to suboptimal outcomes. We show that this same suboptimality also applies to LLMs trained in such settings—and validate this insight empirically—highlighting the importance of accounting for action transformation when designing alignment objectives for LLM-based agents.


Formally, an MAMDP consists of a 6-tuple \(\mathcal{M}_f = (\mathcal{S}, \mathcal{A}, P_S, P_A, R, \gamma)\), or equivalently, the 5-tuple of a standard MDP with additional parameter \( P_A \). The state space \((s \in)~ \mathcal{S}\) represents the dialogue history \(\mathcal{H}_t\) as token sequences terminating at timestep $t$, the action space \((a \in)~\mathcal{A}\) contains candidate actions (utterances in the dialogue) sampled from an underlying distribution, and the state transition function \(P_S\) is deterministic~\citep{rafailov2024r}. Now assume a \agent{intervention} \(\pi^I_\theta\) (an LLM with parameters \(\theta\)). \(P_A(a | \pi^I, s)\) represents the probability that \(\pi^I\) selects action $a$ in state $s$, the reward \(R(s, a)\) is an expected utility, and discount factor \(\gamma = 1\). Additionally assume a set of \agent{collaborator}s \(\pi^C\), each of which may be a human, a distribution representing human behavior, or, as in this paper, another LLM optimized in a standard or black-box fashion to be a robust generator of human-like utterances and actions.

Language is inherently ambiguous. Its impact on the world must be filtered through the perspectives of others besides the speaker. Therefore, even a single alternative interpretation of a linguistic intervention by one collaborator can alter its pragmatic force and thus impact on the dialogue. In other words, LLMs' very medium makes action transformation the norm, not the exception. Consider the following example:

\vspace{2mm}
\begin{example}[\textbf{Action Modification in DeliData Wason Card Task}]
\label{ex:deli_mamdp_adversarial_example} 
The Wason Card Selection task is a well-known cognitive puzzle from the 1960s \cite{wason1968reasoning}, wherein subjects are presented with a set of cards each showing a letter or a number (e.g., \(\{D, A, 4, 7\}\)), and have to decide what the {\it minimum} set of cards is that must be flipped over in order to test a rule such as: \textit{\textbf{All cards with vowels on one side have an even number on the other.}} In the task as collected in the DeliData dataset \citep{karadzhov2023delidata}, groups perform this task collaboratively. 
Each player comes up with a solution individually and the group then deliberates to come to a consensus. In this example, the correct solution is to flip $A$ and $7$; this would establish if $A$'s reverse is an even number, as well as the contrapositive---if $7$ has a vowel, it is an example of $Vowel \wedge \neg Even$, which falsifies $Vowel \rightarrow  Even$. Two participants' initial solution might be to flip only $A$ while the other proposes flipping $4$.\footnote{Flipping $4$ provides no information as a card with $4$ on one side cannot exemplify $Vowel \wedge \neg Even$.} In this setting, the dialogue history is the state \(s\), the \agent{intervention}’s proposed intervention (or action) is \(a\), and the collaborators’ reinterpretation is the transformation \(P_A\). Suppose that with its underlying Bellman-optimal policy, \(\pi^I\)'s $a^I_t$ proposes flipping an odd-numbered card, and does not explicitly state that the even-numbered should {\it not} be flipped. In the MAMDP setting, the collaborator $\pi^C$ responds with an action $a^C_t$ that interprets the semantics of $a^I_t$, either faithfully or with some modification, such as checking $A$, $7$ {\it and} $4$. Under action modification, this interpretation, if presented as the consensus solution, is no longer fully correct, as \{$A$, $7$, $4$\} is not the minimum set of cards that would test the rule. 
\end{example}

\paragraph{Theoretical Insights.}
The above illustration already shows the core risk: an intervention that is Bellman-optimal for the unmodified action space can be counter-productive once collaborators reshape it.
Specifically, current algorithms like Direct Preference Optimization (DPO;~\cite{rafailov2024direct}) and Identity Preference Optimization (IPO;~\cite{azar2024general}) satisfy Bellman optimality conditions and have policy structures that retain the optimal policy formulation. We can show how they are suboptimal for collaborative settings because they disregard modifications made to the action space by $\pi^C$, and RL policies lose optimality guarantees when their actions are modified \citep{langlois2021rl}.
\begin{theorem}[$\Psi$-Preference Optimization in Collaborative MAMDPs]
\label{theorem:general_preference_mamdp}
Let $\Psi: [0,1] \rightarrow \mathbb{R}$ be any non-decreasing function and $\beta > 0$ be a temperature parameter. 
Let $P_A(a|s,\pi^I) = \sum_{a' \in A} \pi^I(a'|s) \cdot \pi^C(a|s,a')$, and represent modifications to the probability distribution over the action space by a collaborator policy $\pi^C$, and let $\pi^I$ be an \agent{intervention} policy trained via $\Psi$-preference optimization in a collaborative MAMDP $\mathcal{M}_f = (\mathcal{M}, P_A)$ with MDP $\mathcal{M}$ and $P_A$ following \cite{langlois2021rl}'s definition. $\pi^I$ satisfies Eq.~\ref{eq:bellman}:
\vspace*{-2mm}
 \begin{align}
 \label{eq:bellman}
\pi^I(a|s) = \frac{\exp(Q^I(s,a)/\beta)}{\sum_{a'}\exp(Q^I(s,a')/\beta)}
\end{align}
where $Q^I$ satisfies the Bellman optimality equation for the underlying MDP $\mathcal{M}$. Thus $\pi^I$ is optimal only when actions are sampled without modification. The Bellman-optimality of $\Psi$PO-aligned $\pi^I$ disregards the collaborator $\pi^C$'s modifications. For MAMDPs with LLMs, this unifies \cite{rafailov2024r}'s derivation of DPO in the token MDP with \cite{langlois2021rl}'s proposition that Bellman-optimal policies do not consider action modifications, and extends it to $\Psi$PO/IPO. See \Cref{proof:general_preference_mamdp_proof} for a detailed proof.
\end{theorem}

This distinction is critical as preference-aligned LLM-based agents get deployed in real-world collaborative settings, such as as "supportive" agents in learning environments~\citep{ganesh2023mind,d2024learning,kumaran2024narrativegenie,perkoff2024crafting}. While any alignment method might optimize for the underlying MDP of an MAMDP, with the attendant suboptimality risks, the specific ways suboptimality manifests may be different from method to method. Therefore, prior to deployment, different alignment techniques must be validated in a realistic setting to determine which are likely to be the most appropriate, beyond an atomized comparison to optimal policy outputs.

\vspace*{-2mm}
\subsection{Collaborative Task Settings}
\label{ssec:tasks}

The two collaborative tasks we used to investigate this phenomenon are: {\bf (1)} the Wason Card Selection task~\citep{wason1968reasoning} as captured in \textbf{DeliData} \citep{karadzhov2023delidata}. This is briefly described in \Cref{ex:deli_mamdp_adversarial_example}. Each dialogue contains 2--6 participants who are presented with 4 cards with a number or letter on them. They must collectively decide which cards to flip to test the rule.
As illustrated in \Cref{ex:deli_mamdp_adversarial_example}, the right answer is to flip a card showing a {\it vowel} and a card showing an {\it odd number}. Participants come up with individual solutions and then deliberate. Utterances are annotated with types of deliberation, allowing us to identify where friction occurs. {\bf (2)} The \textbf{Weights Task}~\citep{khebour2024text},
in which triads deduce the weights of differently-colored blocks with the aid of a balance scale. The correct weight values are $red=10g,~blue=10g,~green=20g,~purple=30g,\text{ and }yellow=50g$. In this multimodal task, participants communicate with language, gestures, and/or actions, and so the data is enriched with friction utterance annotations, and annotations of gestures, actions, and their meanings. 

\vspace*{-2mm}
\subsection{How Do We Train An Intervention Agent?}
\label{ssec:how_to_train_methods}

\paragraph{Data Generation.}
Naturally-occurring friction in collaborative task datasets is sparse, which limits the search space of possible outcomes for a model trained only over real data.\footnote{For instance, "probing" interventions, the chief instance of friction in the DeliData dataset, occurs at a rate of only 3.46 interventions per group, out of 17,110 total utterances (500 groups).} This motivated the first of two uses of the {\bf roleplay simulation} approach \citep{li2023camel,shani2024multiturnreinforcementlearningpreference} central to this work, to simulate diverse language-using agent behavior for data and evaluation needs. For data generation, following \citet{li2023camel}, a single expressive policy can be used to roleplay multiple individuals with appropriate prompting, and LLM roleplays of multi-agent natural language dialogue and reasoning behavior have been shown to have high correlation with human labels \citep{wiegreffe2021measuring,jiang2023evaluating}.

We collected dialogue trajectories in the two tasks described in Sec.~\ref{ssec:tasks} (hereafter referred to as \textit{DeliData} and \textit{WTD}) as roleplays between an Oracle agent $\mathcal{O}$ acting as the \agent{intervention} and a \agent{collaborator} $\pi^C$ that roleplayed all task participants. During data generation, we used off-the-shelf GPT-4o~\citep{openai2024gpt4ocard} as a high-capacity LLM to simulate both types of agent. Roleplay began with a set of 
task-specific guidelines. Every turn consisted of a {\bf back-and-forth interaction} between the simulated agents. Fig.~\ref{fig:roleplay_full_high_level}[L,C] shows a high-level schematic. The Oracle's role as the \agent{intervention} was to track the dialogue, identify frictive states in the dialogue in terms of impasses or breakdowns in common ground, and \textit{intervene} to prompt for reflection and deliberation on those items of confusion. The collaborator then continued the interaction as all task participants.\footnote{The number of participants roleplayed by the collaborator during the data generation phase varies based on the task: for WTD, the number is fixed at 3; for DeliData, the number may be between 2--6, with an average of 4.3.}

Specifically, at each turn $t$ of a dialogue, the oracle identified the current frictive state \(\phi_t\). Then, it generated \(K\) candidate interventions \(\{f_j\}_{j=1}^K\) conditioned on the dialogue state \(s_i\) and frictive state \(\phi_t\). The \agent{collaborator} \(\pi^C\) generated a response \(c_j\) to each candidate intervention.\footnote{Note that $c_j$ can represent more than one simulated participant's utterance to allow for multiple speaking turns. In our experiments, the collaborator was explicitly guided to generate one utterance per turn for each participant in the simulated group, where each participant had a personality trait sampled from a pre-collected pool~\citep{wang2022self} to increase the diversity of simulated behaviors.} Consistent with the MAMDP framework, these responses may have \textit{modified, reinterpreted, or disregarded} the intervention's intent or semantic content (see Sec.~\ref{ssec:collaborative_mamdp}), as the roleplay prompt (Fig.~\ref{fig:wason_unified_prompt_items}) instructed the collaborator to incorporate the intervention ``{\it if relevant}''. Using "self-rewarding"~\citep{yuan2024self} the collaborator simultaneously scored each intervention between 1 (worst) and 10 (best), quantifying its effect on task progress toward a solution. The highest and lowest rated interventions, \(f_w\) and \(f_l\), were selected using West-of-N~\citep{pace2024westofnsyntheticpreferencesselfimproving} sampling. We recorded these as a winner/loser pair \((f_w, f_l)\) in a {\it preference dataset} \(\mathcal{D}_\text{pref}\) with the associated dialogue state \(s_i\) and frictive state \(\phi_t\). The full turn trajectory was recorded to a {\it trajectory dataset} \(\mathcal{D}_\text{traj}\) where each sample consisted of \(s_i\), \(\phi_t\), and \(f_w\). \(f_w\) and the collaborator's response \(c_j\) were appended to the dialogue state. This process continued for \(N = 15\) turns. See \Cref{appendix:roleplay_prompts} for prompting strategy, and~\Cref{alg:friction_data_generation_and_training} for implementation details.
The generated DeliData dialogues includes chat-style text only, while WTD dialogues may include actions/gestures written out as ``stage directions.''

We use the 400 bootstrap dialogues from the training set of DeliData~\citep{karadzhov2023delidata} for training to collect  $\mathcal{D}_\text{traj}$ and $\mathcal{D}_\text{pref}$. This process resulted in 6,000 preference pairs (15 turns for each dialogue), after which we applied a rule-based mapping to further augment the training data to a scale similar to that of common preference alignment datasets such as Ultrafeedback~\citep{ultrafeedbackBinarizedCleaned}, which is required to train an 8B-scale model without overfitting. In particular, we applied a consistent category-preserving mapping where vowels\footnote{We did not replace consonants since the nature of the Wason Card Selection task ensures that vowels are more prevalent in the original DeliData.} $v \in \{A,E,O,U\}$ were replaced with randomly sampled vowels, even numbers with other even numbers, and odd numbers with other odd numbers. This maintains the logical structure of the Wason Card Task—if "A" and "6" are replaced with "E" and "8", the underlying reasoning remains valid. Applying this mapping to all components $(x, \phi, f_w, f_l)$ expanded our dataset to 68,618 preference pairs. The average scores\footnote{Note that these scores are reported from post step 12 and 13 in \Cref{alg:friction_data_generation_and_training} since these average scores are from the phase before the mapping based augmentation.} (out of 1-10) for the preferred and dispreferred interventions assigned by GPT-4o are 8.03 and 3.96 respectively.   

For Weights Task (WTD)~\citep{khebour-etal-2024-common}, since the original data is textually sparse and has very few naturally occurring friction interventions, we use our data-generation pipeline (\Cref{alg:friction_data_generation_and_training}) for creating training data for our experiments. Specifically, to reflect the scale of Ultrafeedback~\citep{ultrafeedbackBinarizedCleaned}, a total of 3,375 combinations of personality-facets (3*5 unique combinations for each participant in a triad) were used to bootstrap this process along with original WTD task-guidelines. As such, we obtained a total of 56{,}689 preference pairs for training after holding out 50 dialogues (approximately 750 single-turn preference pairs) for validation sets~\footnote{Note that all our evaluation in our reported experiments use only the bootstrap dialogues from these 50 dialogues of WTD, due to the nature of collaborative task which requires multiturn processes.}. On average, preferred interventions received scores of $8.48 \pm 1.52$ (on a Likert scale of 1-10) on the training set and $8.51 \pm 1.50$ on the test set, while dispreferred interventions scored $6.01 \pm 0.88$ (train) and $6.08 \pm 0.87$ (test), indicating a stable preference gap across both splits. See Fig.~\ref{fig:gpt_collaborator_prompt} for the relevant scoring prompt.

\paragraph{Human Validation of Oracle Outputs.}
We conducted a focused preference-based human evaluation of Oracle outputs, following a standard framework of evaluating human alignment with LLMs \cite{wiegreffe2021measuring}, to justify the Oracle-generated preference labels. {\bf 50 pairs of candidate interventions each} were randomly sampled from the Oracle-generated DeliData and WTD evaluation sets. Each sample involved a preferred and dispreferred friction intervention, where the preference rating was sourced from GPT-4o using self-rewarding. Two human annotators—both male, fluent English-speaking college undergraduates—were asked to select the intervention they believed was more likely to prompt reflection or advance the group’s reasoning, without being shown the correct solution for the task. Our results suggest that there is a strong annotator agreement on the preferred interventions: Cohen’s $\kappa = 0.92$ on DeliData sample and $\kappa = 0.58$ on Weights Task samples. This shows that LLMs with task-specific prompting generate collaborative task-relevant outputs that align closely with human judgments and capture meaningful aspects of human-centered collaboration, rather than merely reflecting synthetic model behavior.

\paragraph{Training.}
An \agent{intervention} should not only help task completion, but also \textit{iteratively} improve common ground by helping resolve topical disagreements. To achieve this, we adopt \citeauthor{nath2025friction}'s {\it Frictional Agent Alignment Framework} (FAAF)~\cite{nath2025friction}. FAAF is an exemplar of policy optimization to incentivize {\it epistemic alignment} through clarifying and deliberative questions and discussion, as proposed by \citet{pustejovsky2025frictive}. It is designed to support collaborative problem solving through friction interventions with a custom training objective that explicitly conditions on the frictive state (denoted \(\phi\)), but has only to date been evaluated in an offline LLM-Judge format.

FAAF optimizes an empirical loss expressed in terms of the differences in two log-ratios:

\vspace*{-2mm}
{\small
\begin{align}
\mathcal{L}_{\text{FAAF}}&= \mathbb{E}_{\mathcal{D}_{\text{pref}}} \left[\left(\frac{1}{2\beta} - (\Delta R + \Delta R')\right)^2\right],
\label{eq:friction_main_equation_short}
\end{align}}
\noindent where $\Delta R = \log \frac{\pi_\theta(f_w \mid s_i, \phi_t)}{\pi_{\text{ref}}(f_w \mid s_i, \phi_t)} - \log \frac{\pi_\theta(f_l \mid s_i, \phi_t)}{\pi_{\text{ref}}(f_l \mid s_i, \phi_t)}$ (the difference in log-ratio between the winning and losing intervention in a sample, with explicit conditioning on the frictive state) and $\Delta R' = \log \frac{\pi_\theta(f_w \mid s_i)}{\pi_{\text{ref}}(f_w \mid s_i)} - \log \frac{\pi_\theta(f_l \mid s_i)}{\pi_{\text{ref}}(f_l \mid s_i)}$ (the implicit reward margin unconditioned on $\phi$). Together the two terms implicitly encode the difference between presence and absence of the frictive state. If, however, we ignore $\Delta R'$ and focus only on the terms that include explicit frictive state conditioning, we arrive at an IPO-like general preference loss, parametrized with $\theta$:

\vspace*{-2mm}
{\small
\begin{align}
\label{eq:friction_general_ipo}
\mathcal{L}_{\text{friction}}(\pi_\theta) &= 
\mathbb{E}_{(s_i, \phi_t, f_w, f_l) \sim \mathcal{D}_\text{pref}}
\Bigg[
\Bigg(
\underbrace{\log \frac{\pi_\theta(f_w \mid s_i, \phi_t)}{\pi_{\text{ref}}(f_w \mid s_i, \phi_t)}}_{\text{implicit win score}} 
- \\
\notag&\underbrace{\log \frac{\pi_\theta(f_l \mid s_i, \phi_t)}{\pi_{\text{ref}}(f_l \mid s_i, \phi_t)}}_{\text{implicit loss score}} 
- \underbrace{\frac{1}{2\beta}}_{\text{margin}}
\Bigg)^2
\Bigg]
\end{align}}

Letting $\Psi: [0,1] \rightarrow \mathbb{R}$ be any non-decreasing function, \(\pi_{\text{ref}}\) be a reference model, and \( \beta \in \mathbb{R}_+ \) be a regularization parameter, Eq.~\ref{eq:friction_general_ipo} is a solution to the inner-max operator of \citeauthor{nath2025friction}'s two-player min-max objective \cite{nath2025friction}: 

{\small
\begin{align}
    \label{eq:two_stage_main_objective_agent}
    \mathcal{J}^*_{\text{FAAF}} &= \min_{\pi^{I'}} \max_{\pi^I} 
   \expec_{\substack{x \sim \rho \\ \phi \sim \pi^{I'}(\cdot \mid x) \\ I \sim \pi^I(\cdot \mid \phi, x)}}
    \Big[ \Psi\big(\mathcal{P}(I \succ I' \mid \phi, x)\big) - \\
    \notag&\beta D_{\text{KL}}(\pi^I \parallel \pi^{\text{ref}} \mid \phi, x) 
    + \beta D_{\text{KL}}\big(\pi^{I'} (\phi |x) \parallel \pi_{\text{ref}} (\phi |x) \big) \Big],
\end{align}}

\noindent
meaning that the FAAF loss with {\it only} the frictive state-conditioning term $\Delta R$ is equivalent to IPO with frictive state-conditioning. Given this, the following lemma holds:

\begin{lemma}[Vanishing Gradient of the Frictive State]
In $\mathcal{L}_{\text{friction}}$ (Eq.~\ref{eq:friction_general_ipo}), the direct contribution of the frictive state $\phi$ to the gradient vanishes when the conditional probability is decomposed. $\mathcal{L}_{\text{FAAF}}$ (Eq.~\ref{eq:two_stage_main_objective_agent}) overcomes this limitation by incorporating marginal terms that preserve gradient information for frictive states. 
See \Cref{lemma:ipo_vanishing_gradient} and \Cref{corollary:friction_plus} in \Cref{ssec:first_part_agent_solution} for proofs.
\end{lemma}



FAAF's $\Delta R'$ incorporates gradients of $\pi_\theta(\phi|x)$, 
acting as a "fall-back" that helps push the model toward the target preference gap $1/2\beta$
(cf. SMAUG~\citep{zhao2023slichfsequencelikelihoodcalibration, pal2024smaugfixingfailuremodes} which retains a fixed margin of implicit rewards). \textit{Thus, we hypothesize that FAAF alignment using both $\Delta R$ and $\Delta R'$ terms improves understanding of \textit{what} makes an important frictive state, rather than just learning how to respond to one.}

We additionally trained intervention agents using the following other approaches besides FAAF: 
\begin{enumerate*}[label=(\arabic*)]
    \item \textbf{Supervised fine-tuning (SFT)}, where \( \pi^I \) was trained directly on expert demonstrations from $\mathcal{D}_\text{pref}$. 
    \item \textbf{Contrastive preference alignment} methods \textbf{DPO}~\citep{rafailov2024direct} and \textbf{IPO}~\citep{azar2024general}, which refined \( \pi^I \) using preference labels from $\mathcal{D}_\text{pref}$, and including the frictive state representation $\phi$ as part of the context, since it is available from the Oracle-generated data. Since training IPO while conditioning on $\phi$ results in a loss identical to $\mathcal{L}_{\text{friction}}$ (Eq.~\ref{eq:friction_general_ipo}), we report results using $\mathcal{L}_{\text{friction}}$ as IPO. 
    \item \textbf{Reinforcement Learning (RL)}, where \( \pi^I \) was fine-tuned via Proximal Policy Optimization (\textbf{PPO};~\cite{schulman2017proximalpolicyoptimizationalgorithms}). We used OPT-1.3B~\citep{zhang2022opt} initialized with the SFT-trained \( \pi^I \) for the reward model (RM) training for PPO (cf. \citep{Hong2024ORPOMP}). 
    \item A \textbf{Behavior-cloned expert} trained directly on filtered trajectories (cf.~\cite{andukuri2024stargateteachinglanguagemodels, song-etal-2024-trial}) from  $\mathcal{D}_\text{traj}$ with no contrastive preference optimization, but including the frictive state $\phi$ collected in $\mathcal{D}_\text{traj}$.
\end{enumerate*}


We used \texttt{Meta-Llama-3-8B-Instruct}~\citep{llama3modelcard}
~as the base model for all trained intervention agents. For all training-related details see \Cref{ssec:experimental_settings_appendix}.

\vspace*{-2mm}
\section{Experiments and Evaluation}
\label{sec:exp}


In addition to the limited search space of interventions inherent in fixed datasets (see Sec.~\ref{ssec:how_to_train_methods}), fixed datasets also lack a principled way to test the long-term effects of novel interventions on the dialogue trajectory. Evaluation benchmarks designed for multi-turn interactions, such as MT-Bench \cite{Bai_2024}, face intrinsic challenges due to the multi-turn nature of the problem. In such benchmarks, there are gold-standard sample answers (e.g., from humans or a high-capacity model such as GPT-4) that showcase desired qualities. Offline multi-turn evaluation gives as context the dialogue/interaction history up to the point at which a model output is is required, and then scores the generated output along the different dimensions by implicitly comparing them to a gold sample. In a realistic setting involving LLM interventions in dialogue, the dialogue continues after the intervention, potentially along a different trajectory than would have occurred without the intervention. Thus, given a context $x$ and two candidate interventions, $f_1$ and $f_2$, the dialogue will very likely contain different subsequent utterances when continuing given $x$ and $f_1$ and when given $x$ and $f_2$, depending on the dialogue participants responses to the respective interventions. Thus subsequent intervention(s) would arrive in divergent dialogue contexts, making fair comparison under identical conditions impossible.

To alleviate this shortcoming in multi-turn evaluation, and to critically examine the evaluation claims made in prior work such as \cite{nath2025friction}, our experimental setup used a roleplay setting similar to that used for data generation (Sec.~\ref{ssec:how_to_train_methods}), with the following key differences:

\begin{enumerate}
    \item Instead of a {\it single} \( \pi^C \) model roleplaying all collaborator agents, each collaborator was simulated by a {\it distinct} instance of a high-capacity LLM.
    \item Friction interventions were generated by the the aligned $\pi^I$ being evaluated instead of by the oracle.
\end{enumerate}

We used individual distinct instances of GPT-4o-mini for the \agent{collaborator}s with temperature $T{=}0$, and top-$p{=}1$. All $\pi^I$ sampling uses $T{=}0$, top-$p{=}0.9$. \Cref{app:example-dialogue} provides an example dialogue.

\paragraph{Evaluation Conditions.}
In each of the two tasks, we run 50 dialogues for 10 turns each. After the initial task description, in every turn the \agent{intervention} proposes an intervention, after which the \agent{collaborator}s continue the game under the actual task rules. For {\bf DeliData}, the data contains initial bootstrapping dialogues, following \citet{karadzhov2023delidata}. For the {\bf Weights Task}, we provide only the task definition for initial grounding in the prompt, as well as only the weight of the red block (10g), following the original setup for collection of human data in this task \cite{khebour2024text}. Thus our evaluation setup faithfully replicates the original study settings of the two tasks.

Under {\it standard} settings, the MAMDP model of collaborative interaction predicts that a \agent{collaborator} can transform the \agent{intervention}'s suggestions before acting upon them, but this may happen irregularly or stochastically. To additionally explicitly test robustness to the suboptimality risks introduced by the MAMDP, we included an {\it explicit MAMDP} setting where the collaborator agent $\pi^C$'s system prompt specifically guided it to verbally acknowledge $\pi^I$'s intervention but not incorporate its suggestions into the next collaborator action (Fig.~\ref{fig:wason_unified_prompt_items}).

\paragraph{Metrics.}
A successful \agent{intervention} in multi-turn, multi-party collaborations should retain an ability to generate interventions that support construction of common ground as well as successful task completion, over the complete duration of the task, even if the collaborator misinterprets or ignores the intent of the interventions. Therefore we focus on metrics that measure the size of common ground---or the number of agreed-upon task-relevant propositions---and that measure the correctness of the solution arrived at by the group. Due to differences in the task specifics, the way we measure these factors differs slightly across the two tasks, but all quantify both {\it set of shared beliefs} and {\it correctness of beliefs}. In the DeliData task we use the following metrics:

\begin{enumerate*}[label=(\arabic*)]
    \item {\bf Common ground size}. This quantifies how many task-relevant propositions the collaborating agents agreed on, and was normalized against the theoretical upper size bound on each task’s propositional space (16 for DeliData), resulting in {\it normalized cumulative common ground (NCCG}).
    \item {\bf Solution accuracy}. This measures how frequently collaborating agents arrived at a correct solution. We also calculated a \textit{fine-grained} score, which allocated 0.25 points each for including target cards (odd numbers, vowels) and excluding irrelevant ones.
    \item {\bf Performance gain}. This is the difference between the average group-level accuracy at the final turn and the average initial accuracy of individual participants from their initial \textit{individual solutions at the start of the collaborative dialogue}.
    \item {\bf Change-of-Mind rate}. Following \cite{karadzhov2022makes}, this describes the fraction of participants who switch stance on at least one card between two consecutive turns.
\end{enumerate*}

Unlike DeliData, which has a best solution (flip a vowel and an odd number) which may be realized differently given the specific cards presented to a given group, the Weights Task has only \textit{one} specific correct solution (see Sec.~\ref{ssec:tasks}) grounded to the weight assignments given to the five blocks. Thus a raw solution accuracy is less meaningful in this task as a single incorrect weight assignment means the whole solution is incorrect. It is more meaningful to focus on the correctness of propositions the group agrees on during the task. Therefore in the Weights Task we use the following metrics:

\begin{enumerate*}[label=(\arabic*)]
    \item {\bf Final common ground size}. This is the total number of propositions in the common ground at the end of the task dialogue. It indicates how many shared beliefs the group accumulated, without accounting for correctness.
    \item {\bf Accuracy-adjusted common ground size}. This adjusts the final common ground size by penalizing incorrect propositions. This metric favors dialogues with fewer errors, and so measures not just the size of the shared belief set, but also how reliable knowledge based on those beliefs is.
    \item {\bf Percentage of incorrect propositions}. This measures the average fraction of {\it incorrect} propositions in the common ground per turn, and directly captures how much of the dialogue’s content is misleading or wrong, giving a sense of reliability (lower numbers being better). However, this metric does not consider the size of the common ground, making it effectively a precision metric, in that a common ground containing only one proposition would have an incorrect percentage of 0\% if that proposition is correct.
\end{enumerate*}

Common ground size at each dialogue step and solution correctness at the end of each dialogue were extracted by GPT-4o using custom detailed task-specific prompts (Figs.~\ref{fig:wason_unified_prompt_items} and \ref{fig:wtd_cg_extraction_prompt}).
Formulas and more detailed definitions for these metrics are given in Appendices~\ref{app:deli_metrics} and \ref{app:wtd_metrics}. They all assessed how well different \agent{intervention}s helped the group build common ground, and how correct the propositions in the common ground at the end of the task were, compared to the correct solutions for each task (see Sec.~\ref{ssec:tasks}).
We aggregated the metrics with means and standard errors.

\begin{table*}
\centering
\footnotesize
\resizebox{\linewidth}{!}{
\begin{tabular}{lccccccccccc}
\hline
\multirow{2}{*}{\textbf{Method}} & \multicolumn{5}{c}{Standard} && \multicolumn{5}{c}{MAMDP} \\
\cmidrule(lr){2-6} \cmidrule(lr){8-12}
& \textbf{Coarse Acc.} & \textbf{Fine Acc.} & \textbf{NCCG} & \textbf{Perf. Gain} & \textbf{CoM Rate} && \textbf{Coarse Acc.} & \textbf{Fine Acc.} & \textbf{NCCG} & \textbf{Perf. Gain} & \textbf{CoM Rate} \\
\cmidrule(lr){1-6} \cmidrule(lr){8-12}
SFT  & $0.355_{\scriptsize{\pm 0.012}}$ & $0.806_{\scriptsize{\pm 0.004}}$ & $0.204_{\scriptsize{\pm 0.002}}$ & $0.244_{\scriptsize{\pm 0.007}}$ & $0.260_{\scriptsize{\pm 0.017}}$ && $0.283_{\scriptsize{\pm 0.012}}$ & $0.702_{\scriptsize{\pm 0.006}}$ & $0.178_{\scriptsize{\pm 0.002}}$ & $0.143_{\scriptsize{\pm 0.010}}$ & $0.310_{\scriptsize{\pm 0.027}}$ \\
PPO   & $0.409_{\scriptsize{\pm 0.010}}$ & $0.767_{\scriptsize{\pm 0.004}}$ & $0.180_{\scriptsize{\pm 0.001}}$ & $0.183_{\scriptsize{\pm 0.008}}$ & $0.322_{\scriptsize{\pm 0.017}}$ && $0.382_{\scriptsize{\pm 0.013}}$ & $0.763_{\scriptsize{\pm 0.006}}$ & $0.181_{\scriptsize{\pm 0.002}}$ & $0.191_{\scriptsize{\pm 0.010}}$ & $0.304_{\scriptsize{\pm 0.022}}$ \\
BC & $0.369_{\scriptsize{\pm 0.009}}$ & $0.812_{\scriptsize{\pm 0.003}}$ & $0.210_{\scriptsize{\pm 0.001}}$ & $0.239_{\scriptsize{\pm 0.007}}$ & $0.267_{\scriptsize{\pm 0.014}}$ && $0.474_{\scriptsize{\pm 0.013}}$ & $0.809_{\scriptsize{\pm 0.006}}$ & $0.179_{\scriptsize{\pm 0.002}}$ & $0.236_{\scriptsize{\pm 0.009}}$ & $0.302_{\scriptsize{\pm 0.018}}$ \\
DPO  & $0.418_{\scriptsize{\pm 0.010}}$ & $0.831_{\scriptsize{\pm 0.003}}$ & $0.209_{\scriptsize{\pm 0.001}}$ & $0.243_{\scriptsize{\pm 0.007}}$ & $0.264_{\scriptsize{\pm 0.013}}$ && $0.428_{\scriptsize{\pm 0.012}}$ & $0.794_{\scriptsize{\pm 0.006}}$ & $0.201_{\scriptsize{\pm 0.002}}$ & $0.224_{\scriptsize{\pm 0.010}}$ & $0.276_{\scriptsize{\pm 0.018}}$ \\
IPO  & $0.352_{\scriptsize{\pm 0.009}}$ & $0.825_{\scriptsize{\pm 0.003}}$ & $0.205_{\scriptsize{\pm 0.001}}$ & $0.246_{\scriptsize{\pm 0.006}}$ & $0.288_{\scriptsize{\pm 0.015}}$ && $0.391_{\scriptsize{\pm 0.012}}$ & $0.774_{\scriptsize{\pm 0.005}}$ & $0.192_{\scriptsize{\pm 0.002}}$ & $0.197_{\scriptsize{\pm 0.010}}$ & $0.272_{\scriptsize{\pm 0.021}}$ \\
FAAF & $0.485_{\scriptsize{\pm 0.010}}$ & $0.851_{\scriptsize{\pm 0.003}}$ & $0.201_{\scriptsize{\pm 0.001}}$ & $0.260_{\scriptsize{\pm 0.007}}$ & $0.270_{\scriptsize{\pm 0.015}}$ && $0.526_{\scriptsize{\pm 0.013}}$ & $0.844_{\scriptsize{\pm 0.005}}$ & $0.196_{\scriptsize{\pm 0.002}}$ & $0.250_{\scriptsize{\pm 0.008}}$ & $0.329_{\scriptsize{\pm 0.025}}$ \\
\hline
\end{tabular}}
\caption{Performance comparison across differently-aligned intervention agents over sampled DeliData collaborative dialogue trajectories. Metrics include coarse-grained accuracy (Coarse Acc.), fine-grained accuracy (Fine Acc.), normalized cumulative common ground (NCCG), performance gain (Perf. Gain), and change-of-mind rate (CoM Rate). Subscripts show standard error of the mean.}
\label{tab:results_deli_combined}
\vspace*{-2mm}
\end{table*}

\begin{table*}[t]
\centering
\small
\begin{tabular}{lccccccc}
\toprule
\multirow{2}{*}{\textbf{Method}} & \multicolumn{3}{c}{Standard} && \multicolumn{3}{c}{MAMDP} \\
\cmidrule(lr){2-4} \cmidrule(lr){6-8}
& \textbf{Final CG} & \textbf{Adjusted CG} & \textbf{Incorrect \%} && \textbf{Final CG} & \textbf{Adjusted CG} & \textbf{Incorrect \%} \\
\midrule
SFT  & $4.267_{\pm 0.142}$ & $3.571_{\pm 0.109}$ & $12.407_{\pm 0.867}$  && $3.920_{\pm 0.095}$ & $3.490_{\pm 0.088}$ & $9.898_{\pm 0.657}$ \\
PPO  & $3.778_{\pm 0.122}$ & $3.252_{\pm 0.100}$ & $6.966_{\pm 0.689}$ && $5.160_{\pm 0.138}$ & $4.504_{\pm 0.123}$ & $13.361_{\pm 0.744}$ \\
BC   & $5.241_{\pm 0.137}$  &  $4.805_{\pm 0.136}$ &  $9.406_{\pm 0.585}$ && $4.167_{\pm 0.092}$ & $3.837_{\pm 0.087}$ & $6.490_{\pm 0.432}$ \\
DPO & $5.714_{\pm 0.139}$ & $4.912_{\pm 0.129}$ & $16.649_{\pm 0.872}$ && $5.760_{\pm 0.144}$ & $5.329_{\pm 0.141}$ & $8.440_{\pm 0.583}$ \\
IPO  & $3.822_{\pm 0.094}$ & $3.294_{\pm 0.088}$ & $14.009_{\pm 0.938}$ && $4.160_{\pm 0.128}$ & $3.635_{\pm 0.103}$ & $6.156_{\pm 0.438}$ \\
FAAF & $5.143_{\pm 0.149}$ & $4.584_{\pm 0.132}$ & $7.111_{\pm 0.401}$ && $8.300_{\pm 0.270}$ & $7.819_{\pm 0.267}$ & $7.837_{\pm 0.486}$ \\
\bottomrule
\end{tabular}
\caption{\label{tab:wtd_combined_results}Performance comparison across differently-aligned intervention agents over sampled Weights Task collaborative dialogue trajectories. Metrics include final common ground size (Final CG), accuracy-adjusted common ground size (Adjusted CG), and mean per-turn error rate (Incorrect \%). Subscripts show standard error of the mean.
}
\vspace*{-2mm}
\end{table*}

\vspace*{-2mm}
\section{Results and Discussion}
\label{sec:results}

Tables~\ref{tab:results_deli_combined} and \ref{tab:wtd_combined_results} show the performance of differently-aligned intervention agents over sampled DeliData and Weights Task collaborative dialogue trajectories, respectively. {\it Standard} and {\it MAMDP} denote the evaluation conditions discussed in Sec.~\ref{sec:exp}.

\vspace*{-2mm}
\subsection{Empirical Findings in the DeliData Task}

Table~\ref{tab:results_deli_combined} shows solution accuracy, common ground size, performance gain, and change-of-mind rate metrics in collaborative task performance of the DeliData Wason Card task under standard and explicit MAMDP settings. In each turn the aligned \agent{intervention} provides interventions to the group which the \agent{collaborator}s individually respond to continue the conversation, and provide their stances on the valid cards for the game.

\paragraph{Common-ground evolution (NCCG).}
Normalized cumulative common ground (NCCG) measures the growth of shared understanding, being the size of the common ground at each dialogue turn, averaged across turns. In the explicit MAMDP setting, the FAAF \agent{intervention} attains a mean NCCG of $19.6\%$, while DPO achieves slightly higher values ($20.1\%$). Although DPO’s higher NCCG suggests faster consensus formation, this acceleration corresponds to less stable interpretive alignment once collaborator-driven action modification is considered, and the agreed-upon propositions are less likely to be correct in the end. In contrast, FAAF’s steadier NCCG trajectory coincides with superior task accuracy and higher change-of-mind rate ($32.9\%$), indicating more reflective, revision-oriented collaboration rather than premature convergence. Moreover, imitation learning baselines such as the BC collaborator model achieve respectable accuracy ($47.4\%$ coarse, $80.9\%$ fine) yet lack the adaptive belief-updating dynamics observed in FAAF. Similarly, DPO, PPO, and IPO—despite being theoretically optimal in standard MDPs—underperform in the modified-action MAMDP setting, as their policies assume direct execution of actions without reinterpretation. Overall, these results suggest that incorporating frictional interventions enables more reliable and deliberative consensus among participants and reinforces the need to explicitly model how communicative actions are transformed within collaborative reasoning environments. Similar dynamics appear in the standard condition, with SFT, DPO, IPO, and BC \agent{intervention}s all slightly exceeding the FAAF \agent{intervention}'s NCCG performance, while lagging in terms of accuracy and performance gain. This indicates that a larger common ground size may be including more incorrect propositions; the FAAF \agent{intervention}'s interventions enable better {\it discrimination} of good vs. bad information by the collaborators. 

\paragraph{Belief revision and task-grounded performance.}
Across models evaluated within the modified-action MDP (MAMDP) setting, we observe distinct patterns in how agents negotiate accuracy, belief revision, and consensus formation. FAAF achieves the highest task performance among all baselines, with coarse-grained accuracy of \textbf{52.6\%} and fine-grained accuracy of \textbf{84.4\%}, surpassing DPO (42.8\% / 79.4\%), PPO (38.2\% / 76.3\%), IPO (39.1\% / 77.4\%), and SFT (28.3\% / 70.2\%). This suggests that FAAF’s alignment strategy yields solutions that remain faithful to the logical invariants of the DeliData Wason Card task even under collaborator-driven reinterpretation. Additionally, this trend highlights a broader point—alignment algorithms that are “theoretically” optimal under standard MDP assumptions—such as DPO and IPO—show degraded performance once actions are subject to modification by collaborators. The corresponding \textit{change-of-mind rates} further illustrate this limitation: while DPO and IPO maintain relatively low rates (27–28\%), they reflect premature convergence rather than sustained deliberation. Models trained under MDP-based assumptions tend to prioritize static agreement over iterative belief adjustment, making them less suited for collaborative reasoning environments that require negotiation and revision over multiple turns. The corresponding \textit{change-of-mind rate} of \textbf{32.9\%} is moderately higher than that of DPO (27.6\%) and IPO (27.2\%), reflecting a model that promotes reflective belief updates and productive reconsideration rather than premature consensus. This pattern aligns with the MAMDP formulation’s motivation: FAAF explicitly models how interlocutors reinterpret interventions, leading to collaborative trajectories that sustain belief revision instead of enforcing single-step agreement. Under the standard condition, PPO actually achieves the highest change-of-mind rate, but this co-occurs with low performance gain and accuracy, indicating instability in participant beliefs with this \agent{intervention}; the collaborators more frequently change stance to an incorrect position, in contrast to FAAF's interventions, which, when they prompt a change of stance, prompt it into a more frequently correct one. When we examine the explicit MAMDP setting specifically, we see that the FAAF \agent{intervention} maintains its high performance across the accuracy and performance gain metrics, improves its positioning on NCCG relative to the other methods, and also achieves the best change-of-mind rate. This suggests a robustness to the MAMDP condition. While change-of-mind rate may not necessarily correlate to more correct solutions (because a collaborator could be persuaded of an incorrect stance), in the MAMDP condition, \agent{collaborator}s are explicitly guided to reinterpret interventions in a way that resists belief update. Increased change-of-mind rate co-occurring with persistently high solution accuracy and performance gain suggests that the FAAF \agent{intervention} can make up for collaborator reinterpretation or disregarding of previous interventions with subsequent interventions. This is supported by the fact that for groups with the FAAF \agent{intervention} in the loop actually achieve {\it higher} metrics of common ground and task correctness in the explicit MAMDP condition.

\vspace*{-2mm}
\subsection{Empirical Findings in the Weights Task}
Table~\ref{tab:wtd_combined_results} reports Final CG (size of common ground at the last turn), Adjusted CG (Final CG discounted by the dialogue’s per-turn error rate), and Incorrect \% (mean per-turn error rate) over the Weights Task dialogues.

\paragraph{Standard setting.}
In the standard setting of the Weights task, the performance patterns reveal how different alignment methods balance common-ground expansion with accuracy preservation (Table~\ref{tab:wtd_combined_results}). The FAAF \agent{intervention} performs robustly even in the absence of consistently enforced collaborator resistance, attaining a \emph{Final CG} of 5.143 and an \emph{Adjusted CG} of 4.584, while maintaining one of the lowest \emph{Incorrect \%} at 7.111. This indicates that FAAF not only supports the formation of a larger shared knowledge space but also ensures that much of this accumulated content remains correct after error adjustment. In contrast, DPO reaches a higher \emph{Final CG} of 5.714 but suffers the largest \emph{Incorrect \%} at 16.649, suggesting that its consensus grows rapidly but incorporates more faulty propositions that fail under scrutiny. One such example might be a common ground that contains $green=20g$ and $green\neq30g$ (mutually consistent) but also $yellow=40g$ (incorrect). This illustrates a classic trade-off—rapid convergence at the cost of correctness. PPO and IPO demonstrate more conservative behavior with small Final and Adjusted CGs, although PPO has a lower error rate. The SFT baseline lags across all metrics, with a low Final CG and Adjusted CG, and a high Incorrect \%, emphasizing the limitations of imitation-only models in multi-turn reasoning. The multiturn BC baselines achieves relatively higher Final and Adjusted CG than SFT, and even FAAF (albeit within the margin of error), since it learns from expert transitions, but this likewise comes at a cost to correctness of propositions. 

Groups guided by FAAF still achieve higher-quality consensus, preserving more correct propositions while expanding their shared understanding. These findings suggest that even when collaborator reinterpretation is not enforced, FAAF naturally stabilizes belief formation and filters out premature or incorrect consensus, capturing a more reliable form of alignment that translates well to downstream reasoning stability.

\paragraph{Explicit MAMDP setting.}
The FAAF \agent{intervention} attains the strongest common ground formation under collaborator modification: it achieves the largest \emph{Final CG} (8.300) and the largest \emph{Adjusted CG} (7.819), indicating that most of what the groups with FAAF in the loop add to common ground remains correct after error discounting. 
It also shows a low to moderate \emph{Incorrect \%} (7.837). DPO forms a solid second tier with \emph{Final CG} of 5.760 and \emph{Adjusted CG} of 5.329, alongside controlled errors (8.440).
Group with PPO accumulate a comparable \emph{Final CG} (5.160) but suffer the highest error rate (13.361), which pulls the \emph{Adjusted CG} down to 4.504, suggesting a moderately large yet errorful common ground. IPO is more conservative (4.160 / 3.635) with the lowest error among the non-FAAF baselines (6.156) but correspondingly smaller shared knowledge. SFT trails across metrics, displaying the same pattern as the standard condition, suggesting limited ability to expand and stabilize common ground in the explicit MAMDP.

These results reveal two failure modes that the FAAF \agent{intervention} avoids: (i) aggressive growth that introduces too many wrong propositions (DPO, PPO), and (ii) cautious growth that remains clean but ends up too small to be useful (IPO, SFT). FAAF combines large \emph{Final CG} with a small Final CG-Adjusted CG gap, meaning it scaffolds growing consensus while retaining accuracy. DPO is close but still admits more errors and thus loses more after error discounting. The \emph{Incorrect \%} profiles in Table~\ref{tab:wtd_combined_results} explain the \emph{Adjusted CG} ordering: PPO's 13.36\% incorrect rate materially erodes its 5.16 \emph{Final CG}, whereas FAAF's low 7.84\% incorrect rate allows most of the propositions in the associated groups' common ground (average size of 8.300) to survive the penalty.

\vspace*{-2mm}
\subsection{Summary}
Overall, the FAAF \agent{intervention} yields the best task accuracy and performance gain of the compared alignment methods. The SFT baseline underperforms overall, reflecting limitations of single-turn imitation in multi-turn collaborative simulations using LLMs, particularly in the MAMDP setting. PPO likewise struggled due a lack of explicit mechanisms for frictive resolution. The behavior-cloned expert (BC) collaborator benefits from multi-turn structure and imitates strong trajectories but does not reliably \emph{cause} robust realignment toward a correct solution, as reflected in the middling accuracy and performance gain numbers. DPO and IPO show larger common ground but higher error rates, or lower error rates but smaller common ground. This indicates that groups with these \agent{intervention}s in the loop either consolidate agreement quickly albeit around more errors, or agree on fewer things overall even when those things are correct. In aggregate, these findings empirically support the theoretical claim that accounting for action modification in collaborative settings is critical to multiagent coordination: FAAF’s friction-aware design indeed appears to better align the dynamics of common-ground formation with correct task outcomes.

It is somewhat surprising that groups may achieve larger average common ground under the MAMDP condition, since collaborators in this setting are explicitly instructed to resist full alignment by acknowledging the intervention while persisting in their prior reasoning. One might expect this friction to fragment shared understanding, yet the opposite occurs. A key reason may lie in how repeated negotiation and reinterpretation drive convergence in dialogue. For example, in the MAMDP condition, collaborator agents naturally revisit and rephrase the intervention’s claims repeatedly across turns, producing iterative refinements of the same propositions (e.g., “purple is heavier than green,” “purple must be more than 20g but less than 50g”). This linguistic redundancy encourages collaborator agents to gradually stabilize a set of shared mutually consistent, yet not necessarily identical beliefs.
As a result participants echo, reformulate, and confirm each other’s statements, which the evaluation captures as accumulated common ground. 

\vspace*{-2mm}
\section{Conclusion and Future Work}
\label{sec:conc}
In this paper, we examined LLM agent interventions to support multi-turn, multi-party collaborative problem solving. Through a Modified-Action MDP model of collaborative tasks, we theoretically motivated why current common alignment methods should not remain reliably optimal over a dialogue where collaborator modifications change the distribution of the action space. We then empirically demonstrated this by training multiple LLM-driven intervention agents using existing methods, and evaluating them in two different collaborative tasks on the task of inserting {\it friction interventions} to scaffold deliberative and critical reasoning in collaborative tasks. We used a \textit{roleplay} evaluation framework to explicitly address shortcomings in certain multi-turn evaluation frameworks, and to allow us to evaluate the suboptimality induced by the MAMDP setting of collaborative tasks. We showed that the FAAF alignment method, specifically designed for friction interventions, indeed outperforms other methods on facilitating a balance of group common ground convergence and correct task solutions. Additionally, it demonstrates robustness to collaborator action modification or resistance to belief update. Our study emphasizes that in multiagent collaboration, as in human-human collaboration, the collaborative process is as important as the outcome.


Our findings suggest that friction, rather than obstructing alignment, can paradoxically deepen it by promoting iterative clarification. The MAMDP condition thus models a realistic kind of coordinated reasoning—one where alignment emerges through tension, repetition, and gradual stabilization rather than one-shot agreement. For future work, it would be valuable to examine whether this effect could be a consequence of exposure bias~\citep{xu2020rethinkingexposurebiaslanguage} to some extent or whether it persists in tasks with larger or more open-ended hypothesis spaces—such as collaborative story generation, multi-agent resource allocation, open-domain causal reasoning~\citep{chevaleyre2005issues,richens2024robust} or decision-making under uncertainty \citep{nath2025dpl, nath2024simultaneous}. In such settings, redundant clarification might no longer suffice, and agents may need to strategically balance friction with exploration to form diverse yet coherent shared beliefs.

To perform a controlled, high-throughput evaluation, we used an LLM roleplay methodology. The next logical step is studying agent interventions with real human subjects, e.g., by reproducing the studies of the Wason task \cite{karadzhov2023delidata} or Weights Task \cite{khebour-etal-2024-common} with the inclusion of a demonstrably-reliable friction intervention agent in a real-time common ground tracking system, e.g., \cite{vanderhoeven2025trace}.

We also produced a data collection and evaluation pipeline that could be used for red-teaming aligned agents before deployment or examining team dynamics in a digital twin setting to validate the reliability of agent behaviors under diverse simulated conditions. We also hope this study raises awareness of the utility of "friction" to prompt deliberation and accountable decision making in multiagent and human-AI systems, and shows that slower AI interactions can also be positive ones.



\begin{acks}
This material is based in part upon work supported by Other Transaction award HR00112490377 from the U.S. Defense Advanced Research Projects Agency (DARPA) Friction for Accountability in Conversational Transactions (FACT) program and by the U.S. National Science Foundation (NSF) under award DRL 2454151. Approved for public release, distribution unlimited. Views expressed herein do not reflect the policy or position of the National Science Foundation, the Department of Defense, or the U.S. Government. Portions of this work were performed on the Colorado State University Data Science Research Institute high-performance computer \textit{Riviera}. Any remaining errors are the responsibility of the authors.
\end{acks}


\bibliographystyle{ACM-Reference-Format} 

\clearpage

\title[Collaborate, Deliberate, Evaluate]{Collaborate, Deliberate, Evaluate: How LLM Alignment Affects Coordinated Multi-Agent Outcomes: Technical Appendix}

\appendix

\section{DeliData Evaluation Metrics}
\label{app:deli_metrics}

\begin{itemize}

\item \textbf{Coarse Accuracy.}  
A binary metric that equals $1$ if the group’s final submission exactly matches the solution set (all vowel cards $\cup$ all odd-number cards), and $0$ otherwise.  
\emph{Intuition:} This captures whether the group fully solved the Wason task, providing a strict measure of success.

\item \textbf{Fine-Grained Accuracy.}  
A score in increments of $0.25$, where points are awarded for (i) turning a vowel, (ii) turning an odd number, (iii) not turning an even number, and (iv) not turning a consonant.  
\emph{Intuition:} This reflects partial correctness and gives credit for reasoning steps that align with the rule, even if the final submission is not fully correct.

\item \textbf{Performance Gain.}  
Defined as the difference between the average group-level accuracy at the final turn and the average initial accuracy of individual participants from their initial \textbf{individual solutions at the start of the collaborative dialogue}:  
\[
\text{Gain} = \frac{1}{|G|}\sum_{g \in G} \text{Acc}_{\text{final}}(g) - \frac{1}{|P|}\sum_{p \in P} \text{Acc}_{\text{initial}}(p).
\]  
\emph{Intuition:} Measures the benefit of collaboration and intervention relative to participants’ solo reasoning.

\item \textbf{Individual Change.}  
For each participant $p$, computed as $\Delta_p = \text{Acc}_{\text{final}}(p) - \text{Acc}_{\text{initial}}(p)$.  
\emph{Intuition:} Shows how much each participant personally improved or declined, revealing heterogeneity in learning from the interaction.

\item \textbf{Common Ground Convergence.}  
The first turn $t$ at which the majority-supported set of cards stabilizes and remains unchanged for the rest of the dialogue.  
\emph{Intuition:} Captures how quickly the group reaches a shared consensus, highlighting efficiency of reasoning.

\item \textbf{Support Entropy.}  
For each turn $t$, the Shannon entropy of stance distributions:  
\[
H_t = - \sum_{c} p_t(c) \log p_t(c),
\]  
where $p_t(c)$ is the fraction of participants supporting card $c$.  
\emph{Intuition:} High entropy indicates diverse, unresolved opinions (exploration), while low entropy reflects convergence toward consensus.

\item \textbf{Majority Stability.}  
The Jaccard similarity between majority-supported sets at consecutive turns:  
\[
J(A_t, A_{t+1}) = \frac{|A_t \cap A_{t+1}|}{|A_t \cup A_{t+1}|}.
\]  
\emph{Intuition:} Measures how stable group consensus is across turns, distinguishing flip-flopping from steady alignment.

\item \textbf{Turn-to-Turn Effect.}  
The $\ell_1$ distance between support distributions across consecutive turns:  
\[
D_t = \sum_{c} |p_t(c) - p_{t+1}(c)|.
\]  
\emph{Intuition:} Quantifies the size of stance shifts per turn, showing how strongly an intervention or dialogue changes opinions.

\item \textbf{Participant Consistency.}  
The Jaccard similarity of a participant’s supported set of cards across their own turns.  
\emph{Intuition:} Indicates whether participants maintain a stable line of reasoning or frequently switch stances.

\item \textbf{Change-of-Mind Rate.}  
The fraction of participants who switch stance on at least one card between two consecutive turns:  
\[
\text{CMR}_t = \frac{|\{p : S_t(p) \neq S_{t-1}(p)\}|}{|P|}.
\]  
\emph{Intuition:} Captures direct responsiveness to interventions by measuring how often participants revise their positions.

\item \textbf{Consensus Realignment.}  
The proportion of turns where the majority support shifts closer to the correct solution set after the intervention.  
\emph{Intuition:} Indicates whether interventions are not only changing opinions, but doing so in a direction that improves collective reasoning.

\end{itemize}

\section{Weights Task Evaluation Metrics}
\label{app:wtd_metrics}

We evaluate the dialogue models with the following metrics. Each metric captures a different facet of how common ground and reasoning quality evolve during interaction.

\begin{itemize}

\item \textbf{Final Total:}  
This is the total number of common ground propositions at the end of a dialogue.  
\emph{Intuition:} It shows how much shared knowledge the group accumulated overall, regardless of whether it was correct or incorrect.

\item \textbf{Accuracy Adjusted Total:}  
This adjusts the final total by penalizing incorrect propositions.  
\emph{Intuition:} It measures not just how much knowledge was built, but how reliable that knowledge is, favoring dialogues with fewer errors.

\item \textbf{Error-Free Relations:}  
This counts only the propositions that are correct, ignoring the incorrect ones.  
\emph{Intuition:} It provides a direct measure of how much accurate common ground was achieved.

\item \textbf{Accuracy-to-Quantity Ratio:}  
This ratio compares the number of error-free propositions to the total propositions produced.  
\emph{Intuition:} It reflects the efficiency of communication—whether participants generated mostly correct contributions or mixed in many errors.

\item \textbf{Error-Weighted Growth:}  
This measures the per-turn growth of common ground, weighted by accuracy.  
\emph{Intuition:} It evaluates how consistently the dialogue added useful knowledge over time, showing whether progress was stable or error-prone.

\item \textbf{Incorrect Percentage:}  
This is the fraction of propositions in the final common ground that are incorrect.  
\emph{Intuition:} It directly captures how much of the dialogue’s output is misleading or wrong, giving a sense of reliability at the end.

\end{itemize}

\begin{itemize}

\item \textbf{Per-turn proposition counts.}
Let $\mathcal{C}=\{\text{equality},\text{inequality},\text{order}\}$ be the categories.
At dialogue turn $t$, let $S_t^{(c)}$ be the set of unique propositions in category $c\in\mathcal{C}$ and $S_t=\bigcup_{c\in\mathcal{C}} S_t^{(c)}$.
Then the per-turn counts are
\[
N_t^{(c)}=\lvert S_t^{(c)}\rvert,\qquad N_t=\lvert S_t\rvert=\sum_{c\in\mathcal{C}} N_t^{(c)}.
\]
\emph{Intuition:} This measures how much shared structure (common ground) has been accumulated by turn $t$, both overall and by type. It rewards introducing \emph{new} agreed facts rather than repeating old ones.

\item \textbf{Cumulative (final) total.}
For a dialogue with $T$ turns, the final cumulative size is
\[
N_{\mathrm{final}} = N_T,\qquad
N_{\mathrm{final}}^{(c)} = N_T^{(c)}\ \ \text{for each }c\in\mathcal{C}.
\]
\emph{Intuition:} This is the amount of common ground the group ends with. Bigger values indicate broader shared understanding at the end of the conversation.

\item \textbf{Per-turn growth and average growth rate.}
Define growth at turn $t\ge 2$ as
\[
G_t = N_t - N_{t-1}.
\]
The average growth rate is
\[
\overline{G} = \frac{1}{T-1}\sum_{t=2}^{T} G_t.
\]
\emph{Intuition:} Growth captures how quickly new common ground is formed. High average growth indicates steady progress rather than early bursts or stagnation.

\item \textbf{Normalized final total.}
Given a task-specific upper bound $N_{\max}$ on attainable propositions,
\[
\widetilde{N}_{\mathrm{final}} = \frac{N_{\mathrm{final}}}{N_{\max}} \in [0,1].
\]
\emph{Intuition:} This rescales outcomes to a common $[0,1]$ range, enabling comparison across tasks or settings with different ceilings.

\item \textbf{Error rate at the end (incorrect percentage).}
Let $\mathcal{S}_{\mathrm{final}}$ be the final proposition set and let $\mathcal{S}_{\mathrm{true}}$ encode the ground truth. Define the set of incorrect final propositions
\[
\mathcal{E} = \{\, s\in \mathcal{S}_{\mathrm{final}} \ : \ s \text{ contradicts } \mathcal{S}_{\mathrm{true}} \,\},
\]
and the incorrect percentage
\[
p_{\mathrm{err}} = 100 \times \frac{\lvert \mathcal{E}\rvert}{\lvert \mathcal{S}_{\mathrm{final}}\rvert}\ \ (\%).
\]
\emph{Intuition:} This quantifies how noisy the final common ground is. Lower values mean the shared beliefs align better with ground truth.

\item \textbf{Error-free relations.}
\[
N_{\mathrm{ok}} = N_{\mathrm{final}} - \lvert \mathcal{E}\rvert.
\]
\emph{Intuition:} This is the count of \emph{useful} (correct) propositions left standing. It separates signal from noise in the final state.

\item \textbf{Accuracy-adjusted total.}
With a task-dependent scaling factor $\alpha>0$ (e.g., $\alpha=1.5$),
\[
N_{\mathrm{adj}} = \alpha\, N_{\mathrm{final}}\left(1-\frac{p_{\mathrm{err}}}{100}\right).
\]
\emph{Intuition:} This rewards building large common ground while penalizing errors. Two systems with the same size get different credit if one is cleaner.

\item \textbf{Accuracy-to-quantity ratio.}
\[
R_{\mathrm{A/Q}} = \frac{N_{\mathrm{ok}}^{\,2}}{N_{\mathrm{final}}}\quad (\text{with }R_{\mathrm{A/Q}}=0\text{ if }N_{\mathrm{final}}=0).
\]
\emph{Intuition:} This favors high precision \emph{and} non-trivial quantity: squaring $N_{\mathrm{ok}}$ rewards being both correct and substantial, while dividing by $N_{\mathrm{final}}$ penalizes bloated, error-prone sets.

\item \textbf{Error-weighted growth.}
\[
G_{\mathrm{ew}} = \overline{G}\left(1-\frac{p_{\mathrm{err}}}{100}\right).
\]
\emph{Intuition:} Progress that comes with fewer errors counts more. It discounts apparent momentum that is built on mistaken beliefs.

\item \textbf{Per-model aggregation (means).}
Given multiple dialogues $d=1,\dots,D$ for a model, any dialogue-level metric $M^{(d)}$ is averaged as
\[
\overline{M}=\frac{1}{D}\sum_{d=1}^{D} M^{(d)}.
\]
\emph{Intuition:} This summarizes a model’s central tendency across conversations, smoothing idiosyncrasies of individual runs.

\item \textbf{Standard error of the mean (SEM).}
Let $\{M^{(d)}\}_{d=1}^{D}$ be dialogue-level values and let $N_{\mathrm{turns}}$ be the total number of turns aggregated for that model (when treating turns as independent trials). Then an SEM estimate is
\[
\mathrm{SEM}(M)=\frac{\sqrt{\frac{1}{D-1}\sum_{d=1}^{D}\bigl(M^{(d)}-\overline{M}\bigr)^2}}{\sqrt{N_{\mathrm{turns}}}}
\]

(or $\sqrt{D}$ in a dialogue-averaging view).

\emph{Intuition:} SEM indicates uncertainty around the model’s mean metric. Smaller SEM means more stable performance across the evidence pooled (dialogues/turns).

\end{itemize}

\section{Proofs}
\label{app:proofs}

\begin{lemma}[Token-Level IPO Equivalence]
In a token-level MDP with deterministic transitions, the policy $\pi_\theta$ trained using $\Psi$-Preference Optimization or IPO~\citep{azar2024general} with  $\Psi = I(\cdot)$ corresponds to an optimal maximum entropy policy: $\pi_\theta(a_t|s_t) = \frac{\exp(Q_\theta(s_t,a_t)/\beta)}{\sum_{a'}\exp(Q_\theta(s_t,a')/\beta)}$, where $Q_\theta$ satisfies the soft Bellman equation: $Q_\theta(s_t,a_t) = r_{\text{IPO}}(s_t,a_t) + \gamma\mathbb{E}_{s_{t+1}}[V_\theta(s_{t+1})]$, where $I(\cdot)$ is the identity-mapping.
\label{lemma:ipo_token_mdp}
\end{lemma}


\begin{proof}
\label{proof:ipo_token_mdp}
We consider a general non-decreasing function $\Psi : [0, 1] \rightarrow \mathbb{R}$, a reference policy $\pi_{\text{ref}} \in \Delta^{\mathcal{X}}_{\mathcal{Y}}$, and a real positive regularisation parameter $\tau \in \mathbb{R}^{\ast}_{+}$. From \cite{azar2024general}, the $\Psi$-preference optimization objective ($\Psi$PO) is:

\begin{equation}
\max_{\pi} \ \mathbb{E}_{x \sim \rho} \ \mathbb{E}_{y \sim \pi(\cdot|x),\ y' \sim \mu(\cdot|x)} \left[ \Psi\left(p^{\ast}(y \succ y' \mid x)\right) \right] - \beta D_{\text{KL}}(\pi \,\|\, \pi_{\text{ref}}).
\label{eq:psi_po}
\end{equation} where $\rho$ is the context distribution, $p^*$ is the general preference distribution, $\pi_\text{ref}$ is the reference policy, $\Psi$ is a general non-decreasing function and $\beta$\footnote{Note: Throughout this proof, we use $\beta$ to consistently denote both the temperature parameter in the softmax policy and the KL divergence regularization strength. These two interpretations are mathematically equivalent in the maximum entropy RL framework. In some referenced works like \citep{azar2024general}, this parameter is denoted as $\tau$, but we maintain $\beta$ for consistency.} is the KL-divergence regularization strength (or the temperature parameter in max-entropy RL;~\cite{ziebart2008maximum}). 

In a token-level MDP formulation, we can reframe Eq.~\ref{eq:psi_po} in terms of states and actions, where each action represents a token choice and states capture context:

\begin{align}
\max_{\pi} \mathbb{E}_{s\sim\rho, a\sim\pi(\cdot|s), a'\sim\mu(\cdot|s)}[\Psi(p^*(a \succeq a'|s))] - \beta D_{KL}(\pi || \pi_\text{ref})
\label{eq:psi_po_tokenMDP}
\end{align}

Notice that for a particular choice of $\Psi$ as the sigmoid-inverse function, the form of the optimal policy satisfying Eq.~\ref{eq:psi_po_tokenMDP} in terms of the optimal soft-Q function follows directly from~\cite{rafailov2024r}. Under this choice of $\Psi$, Eq.~\ref{eq:psi_po} simply maximizes the reward function in the general MaxEnt RL setting~\citep{ziebart2008maximum, peng2019advantageweightedregressionsimplescalable}. 

\begin{equation}
\pi_\theta(a_t \mid s_t) = \frac{\exp\left(Q^*(s_t, a_t)/\beta\right)}{\sum_{a_t' \in \mathcal{A}} \exp\left(Q^*(s_t, a_t')/\beta\right)}.
\end{equation}

For the general case---where $\Psi$ represents arbitrary non-decreasing function---the equivalence is non-trivial. Specifically, we will \textit{only} consider the case where $\Psi$ is the identity-function, as originally formulated~\citep{azar2024general}. Let us begin with the original IPO loss:

\begin{equation}
L_{IPO}(\pi, D) = \mathbb{E}_{(y^w,y^l)\sim D}\left[\left(h_\pi(y^w, y^l) - \frac{\beta^{-1}}{2}\right)^2\right]
\end{equation}

where $h_\pi(y, y')$ is defined as:

\begin{equation}
h_\pi(y, y') = \log\left(\frac{\pi(y)\pi_{ref}(y')}{\pi(y')\pi_{ref}(y)}\right)
\label{eq:h_log_ratio}
\end{equation}

Now, while the structure of $h_\pi(y, y')$ might be familiar to the reader as the implicit reward advantage~\citep{rafailov2024direct} (ignoring scaling terms like $\beta$), this form does not directly provide us meaningful information of the advantage at the token-level. Therefore, let us first express the responses $y$ and $y'$ in terms of two arbitrary trajectories $\tau = \{s_0^w, a_0^w, ..., s_{N-1}^w, a_{N-1}^w\}$ and $\tau' = \{s_0^l, a_0^l, ..., s_{M-1}^l, a_{M-1}^l\}$, without considering any preference ranking between them. Now, for these complete trajectories, we can rewrite the log-likelihood ratio or the LHS of Eq.~\ref{eq:h_log_ratio} as follows:

\begin{align}
h_\pi(\tau^w, \tau^l) &= \log\left(\frac{\pi(\tau^w)\pi_\text{ref}(\tau^l)}{\pi(\tau^l)\pi_\text{ref}(\tau^w)}\right) \notag \\
&= \log\left(\frac{\prod_{t=0}^{N-1}\pi(a_t^w|s_t^w) \cdot \prod_{t=0}^{M-1}\pi_\text{ref}(a_t^l|s_t^l)}{\prod_{t=0}^{M-1}\pi(a_t^l|s_t^l) \cdot \prod_{t=0}^{N-1}\pi_\text{ref}(a_t^w|s_t^w)}\right) \notag \\
&= \log\left(\prod_{t=0}^{N-1}\frac{\pi(a_t^w|s_t^w)}{\pi_\text{ref}(a_t^w|s_t^w)}\right) - \log\left(\prod_{t=0}^{M-1}\frac{\pi(a_t^l|s_t^l)}{\pi_\text{ref}(a_t^l|s_t^l)}\right) \notag \\
&= \sum_{t=0}^{N-1}\log\frac{\pi(a_t^w|s_t^w)}{\pi_\text{ref}(a_t^w|s_t^w)} - \sum_{t=0}^{M-1}\log\frac{\pi(a_t^l|s_t^l)}{\pi_\text{ref}(a_t^l|s_t^l)}
\label{eq:log_ratio_ipo}
\end{align}
From \cite{rafailov2024r}, we know that in the token-level MDP for the general max-entropy RL setting, the optimal policy $\pi^*$ under soft Q-learning satisfies:

\begin{equation}
\pi^*(a_t \mid s_t) = \exp\left(\frac{Q^*(s_t, a_t) - V^*(s_t)}{\beta}\right),
\end{equation}

where $Q^*$ is the optimal Q-function, $V^*$ is the optimal value function, and $\beta$ is the temperature parameter.

This formulation also holds for policies optimal under Eq.~\ref{eq:psi_po_tokenMDP} for the case with identity mapping $\Psi = I(\cdot)$, since the optimal policy $\pi^*$ in terms of the reference policy takes a similar structure:

\begin{equation}
\pi^*(\tau \mid x) \propto \pi_{\text{ref}}(\tau \mid x) \exp\left(\frac{\mathbb{E}_{\tau' \sim \mu(\cdot \mid x)}[p(\tau \succ \tau')]}{\beta}\right)
\label{eq:psi_optimal_policy_mdp}
\end{equation}
 
Our core insight here is to notice that unlike the standard token-level RLHF maximum-entropy objective where actions are sampled from the policy itself to compute the reward, the optimal policy in above equation (with $\Psi = I(\cdot)$) samples trajectories directly from the behavior policy, $\mu$. Indeed, the structure of the optimal policy remains consistent for both these objectives and LLMs-as-policies can always be represented as a soft-Q function for some reward function~\citep{zhang2025directvalueoptimizationimproving}, where in this case the reward is the \textit{preference} over an alternate trajectory.

Similarly, for the reference policy, we can express:


\begin{equation}
\pi_\text{ref}(a_t \mid s_t) = \exp\left(\frac{Q_\text{ref}(s_t, a_t) - V_\text{ref}}{\beta}\right),
\end{equation}

We can log-linearize these two forms to derive:

\begin{align}
\log\frac{\pi^*(a_t|s_t)}{\pi_\text{ref}(a_t|s_t)} &= \frac{Q^*(s_t,a_t)-V^*(s_t)}{\beta} - \frac{Q_\text{ref}(s_t,a_t)-V_\text{ref}(s_t)}{\beta} \\
\notag&= \frac{1}{\beta}(Q^*(s_t,a_t) - Q_\text{ref}(s_t,a_t) - V^*(s_t) + V_\text{ref}(s_t))
\end{align}

From the Bellman equation (Eq. 7) in \cite{rafailov2024r}, for any arbitrary non-terminal step $s_{t+1}$, we have:

\begin{equation}
Q^*(s_t, a_t) = r(s_t, a_t) + \beta\log\pi_\text{ref}(a_t|s_t) + V^*(s_{t+1})
\end{equation}

And similarly, in the case of the reference model for $Q_\text{ref}$, we can write:

\begin{equation}
Q_\text{ref}(s_t, a_t) = r_\text{ref}(s_t, a_t) + \beta\log\pi_\text{ref}(a_t|s_t) + V_\text{ref}(s_{t+1})
\end{equation}

Substituting these into our log-ratio:

\begin{align}
\log\frac{\pi^*(a_t|s_t)}{\pi_\text{ref}(a_t|s_t)} &= \frac{1}{\beta}(r(s_t, a_t) + \beta\log\pi_\text{ref}(a_t|s_t) + V^*(s_{t+1}) \\
\notag&\quad - r_\text{ref}(s_t, a_t) - \beta\log\pi_\text{ref}(a_t|s_t) - V_\text{ref}(s_{t+1}) \\
\notag&\quad - V^*(s_t) + V_\text{ref}(s_t)) \\
\notag&= \frac{1}{\beta}(r(s_t, a_t) - r_\text{ref}(s_t, a_t) + V^*(s_{t+1}) - \\
\notag&\quad V_\text{ref}(s_{t+1}) - V^*(s_t) + V_\text{ref}(s_t))
\end{align}

Since we want to express this in terms of the reward difference between the optimal and reference policies, we can define $\Delta r(s_t, a_t) = r(s_t, a_t) - r_\text{ref}(s_t, a_t)$ and $\Delta V(s_t) = V^*(s_t) - V_\text{ref}(s_t)$. This gives us:

\begin{align}
\log\frac{\pi^*(a_t|s_t)}{\pi_\text{ref}(a_t|s_t)} &= \frac{1}{\beta}(\Delta r(s_t, a_t) + \Delta V(s_{t+1}) - \Delta V(s_t))
\end{align}

For a complete trajectory, summing over all token positions and using a telescopic series formulation \citep{gunderson2010handbook}, we find:

\begin{align}
\sum_{t=0}^{N-1}\log\frac{\pi^*(a_t|s_t)}{\pi_\text{ref}(a_t|s_t)} &= \frac{1}{\beta}\sum_{t=0}^{N-1}(\Delta r(s_t, a_t) + \Delta V(s_{t+1}) - \Delta V(s_t)) \\
\notag&= \frac{1}{\beta}\left(\sum_{t=0}^{N-1}\Delta r(s_t, a_t) + \Delta V(s_N) - \Delta V(s_0)\right)
\label{eq:ipo_bellman}
\end{align}

Now, we can represent $h_\pi(\tau^w, \tau^l)$ from Eq.~\ref{eq:log_ratio_ipo} directly in terms  policy log ratios to cumulative reward differences as follows:

\begin{align}
h_\pi(\tau^w, \tau^l) &= \sum_{t=0}^{N-1}\log\frac{\pi(a_t^w|s_t^w)}{\pi_\text{ref}(a_t^w|s_t^w)} - \sum_{t=0}^{M-1}\log\frac{\pi(a_t^l|s_t^l)}{\pi_\text{ref}(a_t^l|s_t^l)} \\
\notag&= \frac{1}{\beta}\left(\sum_{t=0}^{N-1}\Delta r(s_t^w, a_t^w) - \sum_{t=0}^{M-1}\Delta r(s_t^l, a_t^l)\right)
\end{align}


The above result and the form of Eq.~\ref{eq:ipo_bellman} shows that the optimal policy under IPO satisfies the soft Bellman equation:
\begin{equation}
Q_\theta(s_t,a_t) = r_{\text{IPO}}(s_t,a_t) + \gamma\mathbb{E}_{s_{t+1}}[V_\theta(s_{t+1})]
\end{equation}

where $r_{\text{IPO}}(s_t,a_t) = r_{\text{ref}}(s_t,a_t) + \Delta r(s_t,a_t) + \beta\log\pi_{\text{ref}}(a_t|s_t)$, and $\Delta r$ represents the reward advantage over the reference policy---calibrated to achieve the target preference gap of $\frac{1}{2\beta}$. This is the main result of our proof.

Interestingly, this result aligns with Theorem 1 from \cite{rafailov2024r}, which establishes that all reward functions consistent with the same preference model induce equivalent policies when expressed in the form of Eq.~\ref{eq:psi_optimal_policy_mdp}. \textit{More importantly, this result suggests the equivalence is satisfied not just for rewards that are optimal under the Bradley-Terry preference model~\citep{BradleyTerry1952}, but also for other equivalence classes of shaped rewards like $r_{\text{IPO}}$ that are derived directly from   general preferences.}

To further derive the final form of the IPO loss, we can continue the argumentation from~\citep{azar2024general} and use an $L_2$-norm-based approach to minimize the difference between this log-likelihood ratio and the target preference gap. As such, assuming we have access to preference annotated winning and losing trajectories ($\tau^w$ and $\tau^l$ respectively) and sampling from the population preferences as a Bernoulli variable and preference symmetry~\citep{munos2023nash}, we get:

\begin{align}
L_{\text{IPO}}(\pi, D) &= \mathbb{E}_{(\tau^w, \tau^l) \sim D} \Bigg[\Bigg(h_\pi(\tau^w, \tau^l) - \frac{1}{2\beta}\Bigg)^2\Bigg] \\
\notag&= \mathbb{E}_{(\tau^w, \tau^l) \sim D} \Bigg[\Bigg( \sum_{t=0}^{N-1} \log \frac{\pi(a_t^w \mid s_t^w)}{\pi_{\text{ref}}(a_t^w \mid s_t^w)} - \\
\notag&\quad\sum_{t=0}^{M-1} \log \frac{\pi(a_t^l \mid s_t^l)}{\pi_{\text{ref}}(a_t^l \mid s_t^l)} - \frac{1}{2\beta} \Bigg)^2\Bigg]
\end{align}

This formulation directly corresponds to the IPO loss, where  $\beta$ (or $\tau$ in the original paper~\citep{azar2024general}) controls both the temperature in the policy and the strength of regularization toward the reference policy.

\end{proof}

 \begin{lemma}[Token-to-Intervention Bellman Completeness]
\label{lemma:token_to_intervention_bellman}
Let $\mathcal{M}_t = (S, A_t, P_t, r_t, \gamma)$ be a token-level MDP and $\mathcal{M}_i = (S, A_i, P_i, r_i, \gamma)$ be the corresponding intervention-level MDP, where each action $a_i \in A_i$ represents a complete friction intervention comprising a sequence of tokens $a_i = (a_t^1, a_t^2, \ldots, a_t^L)$. 

Assuming token-level Bellman completeness holds~\citep{sutton2018reinforcement, zhou2024archer} for function class $\mathcal{F}$, i.e., for any policy $\pi$ and any function $f \in \mathcal{F}$, there exists $f' \in \mathcal{F}$ such that $\|f'(s, a_t) - T^{\pi}f(s, a_t)\|_{\infty} = 0$ where $T^{\pi}$ is the Bellman operator.

Then, the optimal policy $\pi^I$ derived via $\Psi$-preference optimization satisfies:

\begin{align}
\pi^I(a_i|s) = \frac{\exp(Q^I(s,a_i)/\beta)}{\sum_{a_i'}\exp(Q^I(s,a_i')/\beta)}
\end{align}

where $Q^I$ satisfies the intervention-level Bellman optimality equation.
\end{lemma}

\begin{proof}
Under the token-level Bellman completeness assumption, for any state $s \in S$ and intervention action $a_i \in A_i$ decomposed into $L$ tokens $a_i = (a_t^1, a_t^2, \ldots, a_t^L)$, the approximation error of the value function is:

\begin{align}
\min_{f' \in \mathcal{F}} &\|f'(s, a_i) - T_i^{\pi}f(s, a_i)\|_{\infty} = \min_{f_1,...,f_L \in \mathcal{F}} \|f_1(s, a_i) - T_t^{\pi}f_2(s, a_i) \\
\notag&+ r(s, a_i) + \gamma^{1/L} \mathbb{E}_{s' \sim P(\cdot|s,a_i), a_t^1 \sim \pi(\cdot|s')}[f_2(s', a_t^1)] \\
\notag&- \gamma^{1/L} \mathbb{E}_{s' \sim P(\cdot|s,a_i), a_t^1 \sim \pi(\cdot|s')}[T_t^{\pi}f_3(s', a_t^1)] + \ldots \\
\notag&+ \gamma^{(L-1)/L} \mathbb{E}_{s' \sim P(\cdot|s,a_i), a_t^{1:L-1} \sim \pi(\cdot|s')}[f_L(s', a_t^{1:L-1})] - r(s, a_i) \\
\notag&- \gamma^{(L-1)/L} \mathbb{E}_{s' \sim P(\cdot|s,a_i), a_t^{1:L-1} \sim \pi(\cdot|s')}[T_t^{\pi}f(s', a_t^{1:L-1})]\|_{\infty} \\
\notag&\leq \min_{f_1,...,f_L \in \mathcal{F}} \|f_1(s, a_i) - T_t^{\pi}f_2(s, a_i)\|_{\infty} \\
\notag&+ \sum_{i=2}^{L} \gamma^{(i-1)/L} \mathbb{E}_{s' \sim P(\cdot|s,a_i), a_t^{1:i-1} \sim \pi(\cdot|s')}[\|f_i(s', a_t^{1:i-1}) - \\
\notag&\quad T_t^{\pi}f(s', a_t^{1:i-1})\|_{\infty}] \leq 0
\end{align}

The last inequality follows from token-level Bellman completeness, which guarantees that for each component function, there exists an element in $\mathcal{F}$ that perfectly represents the Bellman update.

This implies that intervention-level Bellman completeness holds, and therefore when $\Psi$-preference optimization is applied at the token level, the resulting policy can be expressed as:

\begin{align}
\pi^I(a_i|s) = \frac{\exp(Q^I(s,a_i)/\beta)}{\sum_{a_i'}\exp(Q^I(s,a_i')/\beta)}
\end{align}

where $Q^I$ satisfies the \textit{intervention-level} Bellman optimality equation, which completes our proof. This result is crucial for our analysis of $\Psi$-Preference Optimization (\Cref{theorem:general_preference_mamdp}) and DPO~\citep{rafailov2024direct} (\Cref{proposition:DPO_mamdp_result}), as it establishes that the soft Q-functions derived from these preference-alignment algorithms at the \textit{token} level maintain their optimality properties at the \textit{intervention} level. This is particularly important in our collaborative MAMDP setting, where both the friction and collaborator agents operate on complete interventions as the standard linguistic unit. Operationally, this allows us to use intervention-level utility or reward measurements for quantifying the quality of friction interventions and their modifications.

\end{proof}

\begin{theorem}[$\Psi$-Preference Optimization in Collaborative MAMDPs]

Let $\Psi: [0,1] \rightarrow \mathbb{R}$ be any non-decreasing function and $\beta > 0$ be a temperature parameter. Any intervention agent policy $\pi^I$ trained via $\Psi$-preference optimization with $\Psi$ as identity-mapping in a collaborative MAMDP $\mathcal{M}_f = (\mathcal{M}, P_A)$, where $P_A(a|s,\pi^I) = \sum_{a' \in A} \pi^I(a'|s) \cdot \pi^C(a|s,a')$ represents modifications by a collaborator policy $\pi^C$, satisfies:
 \begin{align}
\pi^I(a|s) = \frac{\exp(Q^I(s,a)/\beta)}{\sum_{a'}\exp(Q^I(s,a')/\beta)}
\end{align}

where $Q^I$ satisfies the Bellman optimality equation for the underlying MDP $\mathcal{M}$, disregarding the collaborator's modifications through $\pi^C$.
\end{theorem}

\begin{proof}
\label{proof:general_preference_mamdp_proof}
From \Cref{lemma:ipo_token_mdp}, we know that a policy trained using $\Psi$-Preference Optimization with Identity mapping in a token-level MDP corresponds to an optimal maximum entropy policy expressible via soft Q-learning. We now extend this result to the collaborative MAMDP~\citep{langlois2021rl} setting.

The general $\Psi$-preference optimization objective~\citep{azar2024general} is originally defined over responses $y$ and $y'$:
\begin{align}
\max_{\pi} \mathbb{E}_{x \sim \rho, y \sim \pi(\cdot|x), y' \sim \mu(\cdot|x)}[\Psi(p^*(y \succ y'|x))] - \beta D_{KL}(\pi || \pi_\text{ref})
\end{align}

In our token-level MDP formulation, we can reframe this in terms of states and actions, where each action represents a token choice and states capture context:
\begin{align}
\max_{\pi^F} \mathbb{E}_{s\sim\rho, a\sim\pi^F(\cdot|s), a'\sim\mu(\cdot|s)}[\Psi(p^*(a \succeq a'|s))] - \beta D_{KL}(\pi^I || \pi_\text{ref}^I)
\end{align}


From \Cref{lemma:ipo_token_mdp}, in a token-level MDP where $\Psi$ is the identity mapping, the corresponding soft Q-learning policy~\citep{zhang2025directvalueoptimizationimproving} takes the following form:
\begin{align} 
Q^I(s,a) = r_\Psi(s,a) + \beta\log\pi_\text{ref}^I(a|s) + \gamma\mathbb{E}_{s'}\left[\max_{a'}Q^I(s',a')\right],
\end{align}
where \( r_\Psi(s,a) \) denotes the reward function under the identity mapping. Now, from Lemma~\ref{lemma:token_to_intervention_bellman}, we know that under the assumption of token-level Bellman completeness, a policy trained via token-level preference optimization preserves optimality properties when extended to \textit{intervention}-level or complete friction interventions. This aligns with findings by \citet{zhang2025directvalueoptimizationimproving}, who demonstrated that when policies are parameterized by logits, grouping tokens into macro-actions preserves both sequence probability and policy structure. \textit{This theoretical foundation is crucial in our MAMDP setting because it allows us to analyze and measure the quality of the friction agent's policy at the intervention level while training occurs token-by-token.}

Now, let us consider the MAMDP action modification function $P_A$, which transforms intended actions according to the collaborator policy $\pi^C$. Refer \Cref{ex:deli_mamdp_adversarial_example} for an intuitive example of this modification.

\begin{align}
P_A(a|s,\pi^I) = \sum_{a' \in A} \pi^I(a'|s) \cdot \pi^C(a|s,a')
\end{align}

The empirical policy affecting the environment is therefore:
\begin{align}
\dot{\pi}^I(a|s) = P_A(a|s,\pi^I) = \sum_{a' \in A} \pi^I(a'|s) \cdot \pi^C(a|s,a')
\end{align}

For the empirical policy $\dot{\pi}^I(a|s) = \sum_{a' \in A} \pi^I(a'|s) \cdot \pi^C(a|s,a')$, we verify it forms a valid probability distribution. Assuming both $\pi^F$ and $\pi^C$ are valid probability distributions, we have:

\begin{align}
\sum_{a \in A} \dot{\pi}^I(a|s) 
\notag&= \sum_{a \in A} \sum_{a' \in A} \pi^I(a'|s) \pi^C(a|s,a') \\ 
\notag&= \sum_{a' \in A} \pi^I(a'|s) \sum_{a \in A} \pi^C(a|s,a') \\ 
\notag&= \sum_{a' \in A} \pi^I(a'|s) \\ 
\notag&= 1,
\end{align}
where we use $\sum_{a \in A} \pi^C(a|s,a') = 1$ for all $s, a'$ and $\sum_{a' \in A} \pi^I(a'|s) = 1$ for all $s$.

However, weight updates on $\pi^I$ based on $L_{\text{IPO}}$ depends  solely on trajectory preferences without accounting for these modifications. The gradient updates to the policy parameters directly optimize the virtual policy $\pi^I$, not the empirical policy $\dot{\pi}^F$.

The Bellman updates never incorporate $P_A$ or $\pi^C$, and the policy optimizes:
\begin{align}
\pi^I(s) = \arg\max_{a} Q^I(s,a)
\end{align}

which satisfies the Bellman optimality equation for $\mathcal{M}$ regardless of the collaborator's modifications.

Therefore, from~\cite{everitt2021agent}, $\pi^I$ is optimal for the underlying MDP $\mathcal{M}$ while being completely unaware of how its actions are modified by the collaborator through $\pi^C$.
\end{proof}

\begin{proposition}[DPO Bellman Optimality in MAMDPs]
A friction agent policy $\pi^I$ trained via DPO in a collaborative MAMDP $\mathcal{M}_f = (\mathcal{M}, P_A)$ satisfies the Bellman optimality objective for the underlying MDP $\mathcal{M}$, thereby ignoring the effect of the collaborator's action modifications $P_A$.
\label{proposition:DPO_mamdp_result}
\end{proposition}

\begin{proof}
We define the collaborative MAMDP where $P_A$ represents the collaborator policy $\pi^C$ that modifies friction interventions:
$P_A(a|s, \pi^I) = \sum_{a' \in A} \pi^I(a'|s) \cdot \pi^C(a|s, a')$.

The DPO objective optimizes the friction policy by minimizing:
\[
\mathcal{L}(\pi_\theta^I, \mathcal{D}) = -\mathbb{E}_{(\tau^w, \tau^l) \sim \mathcal{D}} \left[\log \sigma \left(\beta \log \frac{\pi_\theta^I(\tau^w)}{\pi_{\text{ref}}^I(\tau^w)} - \beta \log \frac{\pi_\theta^I(\tau^l)}{\pi_{\text{ref}}^I(\tau^l)}\right)\right].
\]
This optimization yields a policy expressible as a softmax over action values:
\[
\pi_\theta^I(a | s) = \frac{\exp(Q_\theta^I(s, a)/\beta)}{\sum_{a'} \exp(Q_\theta^I(s, a')/\beta)}
\]
where $Q_\theta^I(s, a) = \beta \log \pi_{\text{ref}}^I(a | s) + r_{\text{pref}}(s, a)$.

The DPO updates implicitly train these Q-values to satisfy:
\[
Q_\theta^I(s, a) = r_{\text{DPO}}(s, a) + \gamma \mathbb{E}_{s' \sim P_S(s, a)} \left[\max_{a'} Q_\theta^I(s', a')\right].
\]
This update rule corresponds exactly to the Bellman optimality equation for $\mathcal{M}$ with reward function $r_{\text{DPO}}(s, a) = r_{\text{pref}}(s, a) + \beta \log \pi_{\text{ref}}^I(a | s)$.

Critically, the DPO optimization process never incorporates $P_A$ or $\pi^C$. The Q-value updates do not account for the friction agent's chosen action $a$ being potentially transformed into $\hat{a} \sim \pi^C(\cdot | s, a)$. While the empirical policy affecting the environment is $\dot{\pi}^I(a | s) = P_A(a | s, \pi^I)$, the DPO updates are based solely on the virtual policy $\pi^I$.

By Proposition 2 of \citet{everitt2021agent}, policies satisfying the Bellman optimality objective for a MAMDP are optimal for the underlying MDP regardless of action modifications. Therefore, $\pi^I$ trained via DPO optimizes for $\mathcal{M}$ while ignoring the collaborator's modifications through $P_A$.
\end{proof}

\begin{lemma}[Token-Level Q-function Equivalence]
In a token-level MDP with deterministic transitions, the LLM logits $l_\theta$ trained using DPO represent an optimal Q-function $Q^*(s, a)$ corresponding to some reward function $r(s, a)$.
\end{lemma}

\begin{proof}
From the Bellman equation in the token-level MDP:
\begin{align}
Q^*(s_t, a_t) &= r(s_t, a_t) + \beta \log \pi_{\text{ref}}(a_t | s_t) + V^*(s_{t+1}).
\end{align}
The optimal policy is then related to $Q^*$ via:
\begin{align}
\pi^*(a_t | s_t) &= e^{(Q^*(s_t, a_t) - V^*(s_t))/\beta}.
\end{align}
Since this corresponds to a softmax over logits $l_\theta$ with temperature $\beta$, and because DPO optimizes these logits to match preference data, it follows that DPO effectively learns a Q-function representation.
\end{proof}

\subsection{Proof of Optimal Friction Intervention Policy}
\label{ssec:first_part_agent_solution}
The structure of this solution follows standard results in RL and control theory literature, appearing in preference alignment frameworks for LLMs~\cite{ziebart2008maximum, peng2019advantageweightedregressionsimplescalable, rafailov2024direct, azar2024general} and CoT-based alignment frameworks~\cite{choi2024robust}. We simply demonstrate that a similar application holds for \citeauthor{nath2025friction}'s FAAF model \cite{nath2025friction} in our collaborative setting where \agent{intervention} is additionally conditioned on the frictive state, $\phi$. This proof follows similar reasoning as in~\citet{azar2024general}. Let us recall the general preference optimization objective for \agent{intervention} in Eq.~\ref{eq:two_stage_main_objective_agent}, assuming $\Psi$ as identity-mapping~\citep{azar2024general}.

{\small
\begin{align}
\mathcal{J}^*_{\text{FAAF}} &= \min_{\pi'} \max_{\pi} 
\mathbb{E}_{\substack{x \sim \rho \\ \phi \sim \pi'(\cdot \mid x) \\ f \sim \pi(\cdot \mid \phi, x)}} \Bigg[
\mathcal{P}(f \succ f' \mid \phi, x) - \beta D_{\text{KL}}(\pi \parallel \pi_{\text{ref}} \mid \phi, x) + \\
\notag&\beta D_{\text{KL}}(\pi' \parallel \pi_{\text{ref}} \mid x) 
\Bigg].
\end{align}
}


For fixed $\pi'$, the inner maximization reduces to the regularized objective:
\begin{align}
\mathcal{L}_\beta(\pi) &= \mathbb{E}_{f\sim\pi}[\mathcal{P}(f\succ f'|\phi, x)] - \beta D_{\text{KL}}(\pi \parallel \pi_{\text{ref}}|\phi,x) \\
\notag&= \sum_{f}\pi(f|\phi,x)\mathcal{P}(f\succ f'|\phi, x) - \beta D_{\text{KL}}(\pi \parallel \pi_{\text{ref}}|\phi,x),
\end{align}

where $f \in \mathcal{F}$ comes from a finite space of friction \textit{interventions}, $\mathcal{P}(f\succ f'|\phi, x)$ provides the preference feedback from collaborator participants, $\beta\in\mathbb{R}^*_+$ is strictly positive, and $\pi, \pi_{\text{ref}}$ are LLM policies. Note that $\pi(f|\phi,x)$ is a valid probability distribution, satisfying:
  
\begin{equation}
   \sum_{f}\pi(f|\phi,x) = 1.
\end{equation}

Let us first define the optimal friction intervention policy $\pi^*$ as:
\begin{equation}
   \pi^*(f|\phi,x) = \frac{\pi_{\text{ref}}(f|\phi,x)\exp(\beta^{-1}p(f\succ f'|\phi, x))}{Z^*(\phi,x)},
\end{equation}

where $Z^*(\phi,x) = \sum_{f'}\pi_{\text{ref}}(f'|\phi,x)\exp(\beta^{-1}p(f'\succ f'|\phi, x))$. Under these definitions:
\begin{equation}
   \pi^* = \underset{\pi}{\operatorname{arg\,max}}\mathcal{L}_\beta(\pi)
\end{equation}

\begin{proof}
\begin{align}
\frac{\mathcal{L}_\beta(\pi)}{\beta} &= \sum_{f \in \mathcal{F}}\pi(f|\phi,x)\frac{p(f\succ f'|\phi, x)}{\beta} - D_{\text{KL}}(\pi \parallel \pi_{\text{ref}}|\phi,x)\\
\notag&= \sum_{f \in \mathcal{F}}\pi(f|\phi,x)\big(\frac{p(f\succ f'|\phi, x)}{\beta} - \log\big(\frac{\pi(f|\phi,x)}{\pi_{\text{ref}}(f|\phi,x)}\big)\big)\\
\notag&= \sum_{f \in \mathcal{F}}\pi(f|\phi,x)\log\big(\frac{\pi_{\text{ref}}(f|\phi,x)\exp(\beta^{-1}p(f\succ f'|\phi, x))}{\pi(f|\phi,x)}\big)\\
\notag&= \sum_{f \in \mathcal{F}}\pi(f|\phi,x)\log\big(\frac{\pi_{\text{ref}}(f|\phi,x)\exp(\beta^{-1}p(f\succ f'|\phi, x))}{Z^*(\phi,x)}\cdot \\ &\frac{Z^*(\phi,x)}{\pi(f|\phi,x)}\big)\\
\notag&= \sum_{f \in \mathcal{F}}\pi(f|\phi,x)\log\big(\frac{\pi^*(f|\phi,x)}{\pi(f|\phi,x)}\big) + \log Z^*(\phi,x)\\
\notag&= -D_{\text{KL}}(\pi \parallel \pi^*) + \log Z^*(\phi,x)
\end{align}

By definition, $\pi^* = \underset{\pi}{\operatorname{arg\,max}}\big[-D_{\text{KL}}(\pi \parallel \pi^*)\big]$. Since:
\begin{equation}
-D_{\text{KL}}(\pi \parallel \pi^*) = \frac{\mathcal{L}_\beta(\pi)}{\beta} - \log Z^*(\phi,x)
\end{equation}

where $\log Z^*(\phi,x)$ is the partition function independent of $\pi$, and $\beta > 0$, the argmax of $-D_{\text{KL}}(\pi \parallel \pi^*)$ coincides with that of $\mathcal{L}_\beta(\pi)$, completing the proof.
\end{proof}

\begin{lemma}[Vanishing Gradient of Frictive State $\phi$]
In $\mathcal{L}_{\text{friction}}$ (Eq.~\ref{eq:friction_general_ipo}), the direct contribution of the frictive state $\phi$ to the gradient vanishes when the conditional probability is decomposed.
\label{lemma:ipo_vanishing_gradient}
\end{lemma}

\begin{proof}
The gradient of the IPO-like $\mathcal{L}_{\text{friction}}(\pi_\theta)$ with respect to $\theta$ is:
\begin{align}
\nabla_\theta \mathcal{L}_{\text{friction}}(\pi_\theta) &= \mathbb{E}_{\mathcal{D}} 
\bigg[
2\delta \cdot \Big( \nabla_\theta \log \pi_\theta(f_w | s, \phi) - \notag\\
&\quad\quad\nabla_\theta \log \pi_\theta(f_l | s, \phi) \Big)
\bigg]
\end{align}

where $\mathcal{D} = \{(s, \phi, f_w, f_l)\}$ is the preference dataset, $\delta = \log \frac{\pi_\theta(f_w | s, \phi)}{\pi_{\text{ref}}(f_w | s, \phi)} - \log \frac{\pi_\theta(f_l | s, \phi)}{\pi_{\text{ref}}(f_l | s, \phi)} - \frac{1}{2\beta}$, and $s, \phi, f_w, f_l$ represent the context, frictive state, winning and losing friction interventions, respectively.

Decomposing the conditional distribution in a standard fashion:
\begin{align}
\log \pi_\theta(f | s, \phi) &= \log \pi_\theta(f, \phi | s) - \log \pi_\theta(\phi | s)
\end{align}

Taking the gradient and applying the linearity of the gradient operator, we get:
\begin{align}
\nabla_\theta \log \pi_\theta(f | s, \phi) &= \nabla_\theta \log \pi_\theta(f, \phi | s) - \nabla_\theta \log \pi_\theta(\phi | s)
\end{align}

The difference of gradients in the objective becomes:
\begin{align}
&\nabla_\theta \log \pi_\theta(f_w | s, \phi) - \nabla_\theta \log \pi_\theta(f_l | s, \phi)\\
\notag&= \nabla_\theta \log \pi_\theta(f_w, \phi | s) - \nabla_\theta \log \pi_\theta(\phi | s) - \\
\notag&[\nabla_\theta \log \pi_\theta(f_l, \phi | s) - \nabla_\theta \log \pi_\theta(\phi | s)] \\
\notag&= \nabla_\theta \log \pi_\theta(f_w, \phi | s) - \nabla_\theta \log \pi_\theta(f_l, \phi | s)
\end{align}

Thus, the $\nabla_\theta \log \pi_\theta(\phi | s)$ terms cancel out, showing that the direct contribution of $\phi$ vanishes in the gradient computation. Note that~\cite{pal2024smaugfixingfailuremodes} and \cite{zhang2024chainpreferenceoptimizationimproving} provides a similar argument to empirically show that DPO~\citep{rafailov2024direct}'s loss suffers from a similar vanishing gradient problem limiting policy learning especially when the preferred and the dispreferred responses or CoT-\textit{trajectories} are highly similar at the string level. These studies show \textit{when} DPO might assign low likelihood to the winning responses, despite the DPO implicit reward margin increasing during training. Subsequently~\cite{rafailov2024r} offers theoretical justification for this phenomenon (reduction in the preferred response likelihood) with the additional insight that this is more likely when the policy first undergoes supervised-finetuning (SFT) and that this is expected from the perspective of the objective (MaxEnt RL in token-MDP)---with similar results seen also in the case of the general MDP~\citep{hejna2024contrastivepreferencelearninglearning}. In contrast, our work extends this observation where additional random variables like frictive states $\phi$ are modeled as a part of the state decomposition in the token-MDP. As such, we extend this observation to learning algorithms like IPO~\citep{azar2024general} that optimize for general preferences.

\end{proof}

\begin{corollary}
The combined loss function $\mathcal{L} = \mathbb{E}_{\mathcal{D}_{\text{pref}}} [(1/2\beta - (\Delta R + \Delta R'))^2]$ incorporating both conditional and marginal terms promotes more effective learning of the friction state gradient compared to the standard friction-IPO loss.
\label{corollary:friction_plus}
\end{corollary}

\begin{proof}

To recall from \Cref{ssec:how_to_train_methods}, our collaborative roleplay results in $\mathcal{D}_{\text{pref}}$---a dataset of tuples $(s, \phi, f_w, f_l)$ where $s$ represents context, $\phi$ is a frictive state, and $f_w, f_l$ are preferred and non-preferred friction interventions, respectively. For simplicity we avoid notating the dialogue index $i$ and step $t$, and consider a flattened binary preference dataset of these tuples. Additionally, let $\Delta R$ and $\Delta R'$ be defined as in \citet{nath2025friction}:

\begin{align}
\Delta R &= \log\frac{\pi_\theta(f_w|\phi,s)}{\pi_{\text{ref}}(f_w|\phi,s)} - \log\frac{\pi_\theta(f_l|\phi,s)}{\pi_{\text{ref}}(f_l|\phi,s)}\\
\Delta R' &= \log\frac{\pi_\theta(f_w|s)}{\pi_{\text{ref}}(f_w|s)} - \log\frac{\pi_\theta(f_l|s)}{\pi_{\text{ref}}(f_l|s)}
\end{align}



Starting with \citeauthor{nath2025friction}'s loss function $\mathcal{L}_{\text{FAAF}}$ \cite{nath2025dpl}:

\begin{align}
\mathcal{L} &= \mathbb{E}_{\mathcal{D}_{\text{pref}}} \left[\left(\frac{1}{2\beta} - (\Delta R + \Delta R')\right)^2\right]
\end{align}

and then taking the gradient with respect to $\theta$, we get:
\begin{align}
\nabla_\theta \mathcal{L} &= \nabla_\theta \mathbb{E}_{\mathcal{D}_{\text{pref}}} \left[\left(\frac{1}{2\beta} - (\Delta R + \Delta R')\right)^2\right] \\
\notag&= \mathbb{E}_{\mathcal{D}_{\text{pref}}}  \left[\nabla_\theta \left(\frac{1}{2\beta} - (\Delta R + \Delta R')\right)^2\right] \\
\notag&= \mathbb{E}_{\mathcal{D}_{\text{pref}}} \left[2\left(\frac{1}{2\beta} - (\Delta R + \Delta R')\right) \cdot \nabla_\theta\left(\frac{1}{2\beta} - (\Delta R + \Delta R')\right)\right] \\
\notag&= \mathbb{E}_{\mathcal{D}_{\text{pref}}} \left[2\left(\frac{1}{2\beta} - (\Delta R + \Delta R')\right) \cdot \left(-\nabla_\theta(\Delta R + \Delta R')\right)\right] \\
\notag&= \mathbb{E}_{\mathcal{D}_{\text{pref}}} \left[-2\left(\frac{1}{2\beta} - (\Delta R + \Delta R')\right) \cdot \nabla_\theta(\Delta R + \Delta R')\right]
\end{align}

We define $\delta' = \frac{1}{2\beta} - (\Delta R + \Delta R')$ for clarity:
\begin{align}
\nabla_\theta \mathcal{L} &= \mathbb{E}_{\mathcal{D}_{\text{pref}}}  \left[-2\delta' \cdot \nabla_\theta(\Delta R + \Delta R')\right] \\
\notag&= \mathbb{E}_{\mathcal{D}_{\text{pref}}}  \left[-2\delta' \cdot (\nabla_\theta\Delta R + \nabla_\theta\Delta R')\right]
\end{align}

Expanding the terms $\nabla_\theta\Delta R$ and $\nabla_\theta\Delta R'$:

For $\nabla_\theta\Delta R$ from  \Cref{lemma:ipo_vanishing_gradient}:
\begin{align}
\nabla_\theta\Delta R &= \nabla_\theta\left[\log\frac{\pi_\theta(f_w|\phi,x)}{\pi_{\text{ref}}(f_w|\phi,x)} - \log\frac{\pi_\theta(f_l|\phi,x)}{\pi_{\text{ref}}(f_l|\phi,x)}\right] \\
\notag&= \nabla_\theta\log\pi_\theta(f_w|\phi,x) - \nabla_\theta\log\pi_\theta(f_l|\phi,x) \\
\notag&= \nabla_\theta\log\pi_\theta(f_w,\phi|x) - \nabla_\theta\log\pi_\theta(\phi|x) \notag\\
&\quad - \nabla_\theta\log\pi_\theta(f_l,\phi|x) + \nabla_\theta\log\pi_\theta(\phi|x) \\
\notag&= \nabla_\theta\log\pi_\theta(f_w,\phi|x) - \nabla_\theta\log\pi_\theta(f_l,\phi|x)
\end{align}

where the $\nabla_\theta\log\pi_\theta(\phi|x)$ terms cancel, resulting in no direct $\phi$ gradient contribution.

For $\nabla_\theta\Delta R'$, we can write:
\begin{align}
\nabla_\theta\Delta R' &= \nabla_\theta\left[\log\frac{\pi_\theta(f_w|x)}{\pi_{\text{ref}}(f_w|x)} - \log\frac{\pi_\theta(f_l|x)}{\pi_{\text{ref}}(f_l|x)}\right] \\
\notag&= \nabla_\theta\log\pi_\theta(f_w|x) - \nabla_\theta\log\pi_\theta(f_l|x)
\end{align}




Now, expanding the marginal probabilities using the law of total probability~\citep{jeffreys1998theory}:
\begin{align}
\pi_\theta(f|x) = \sum_{\phi'}\pi_\theta(f,\phi'|x) = \sum_{\phi'}\pi_\theta(f|\phi',x)\pi_\theta(\phi'|x)
\end{align}

We then take the gradient to derive:

\begin{align}
\nabla_\theta\log\pi_\theta(f|x) &= \frac{\nabla_\theta\pi_\theta(f|x)}{\pi_\theta(f|x)} \\
\notag&= \frac{1}{\pi_\theta(f|x)}\nabla_\theta\sum_{\phi'}\pi_\theta(f|\phi',x)\pi_\theta(\phi'|x) \\
\notag&= \frac{1}{\pi_\theta(f|x)}\sum_{\phi'}\Big[\pi_\theta(f|\phi',x)\nabla_\theta\pi_\theta(\phi'|x) \notag\\
&\quad + \pi_\theta(\phi'|x)\nabla_\theta\pi_\theta(f|\phi',x)\Big]
\end{align}

Unlike in the first term, the gradients $\nabla_\theta\pi_\theta(\phi'|x)$ do \textit{not} cancel out. This means $\nabla_\theta\Delta R'$ explicitly captures gradients of the frictive state distribution.

Combining both terms in the loss gradient, we can represent the gradient expression for $\mathcal{L}_{\text{FAAF}}$ as:

\begin{align}
\nabla_\theta \mathcal{L} &= \mathbb{E}_{\mathcal{D}_\mu} \left[-2\delta' \cdot (\nabla_\theta\Delta R + \nabla_\theta\Delta R')\right] \\
&\notag= \mathbb{E}_{\mathcal{D}_\mu} \Bigg[-2\delta' \cdot \Bigg(\underbrace{\nabla_\theta\log\pi_\theta(f_w,\phi|x) - \nabla_\theta\log\pi_\theta(f_l,\phi|x)}_{\nabla_\theta\Delta R} \\
&\notag\quad + \underbrace{\nabla_\theta\log\pi_\theta(f_w|x) - \nabla_\theta\log\pi_\theta(f_l|x)}_{\nabla_\theta\Delta R'}\Bigg)\Bigg]
\end{align}



Where $\delta' = \frac{1}{2\beta} - (\Delta R + \Delta R')$ and the gradient of the marginal terms $\nabla_\theta\log\pi_\theta(f|x)$ includes direct contributions from the frictive state $\phi$ through the weighted sum of $\nabla_\theta\pi_\theta(\phi'|x)$ terms. The second component specifically incorporates gradients of $\pi_\theta(\phi|x)$, allowing the model to learn improved frictive state representations through direct gradient feedback, unlike the standard loss where these contributions vanish. Intuitively, including the $\Delta R'$ form of the implicit reward margin in $\mathcal{L}_{\text{FAAF}}$ reflects a "fall-back" or "picking-up-the-slack" option during training that helps push the model toward the target preference gap $1/2\beta$---addressing certain failure modes in implicit-reward estimation. The preference gap can of course be data-dependent and can be picked optimally during model validation. But the idea of fallback options to avoid such failure modes has been found to be empirically viable, similar to methods like SMAUG~\citep{zhao2023slichfsequencelikelihoodcalibration, pal2024smaugfixingfailuremodes} which penalizes the model to retain a fixed-margin of implicit rewards. Therefore, in training the \agent{intervention} with $\mathcal{L}_{\text{FAAF}}$, the model improves its understanding of \textit{what} makes a viable frictive state, rather than just learning how to respond appropriately, given a frictive state.

\end{proof}

\section{Friction Agent Training Algorithm}

\Cref{alg:friction_data_generation_and_training} shows the \agent{friction} data generation and training algorithm. \Cref{tab:personality_facet_descriptions} shows the personality facets that were ascribed to different roleplay participants by $\pi^C$.

\begin{algorithm}[h!]
\caption{Preference Data Generation and Training \agent{friction} }
\label{alg:friction_data_generation_and_training}
\begin{algorithmic}[1]
\Require Oracle agent $\pi^O$, Collaborator agent $\pi^C$, Bootstrap dialogues \(\mathcal{D} = \{d_i\}_{i=1}^M\), Personality-facet combinations \(\mathcal{P}\), Max turns \(N\), Reference model (SFT) \(\pi_{\text{ref}}\)
 
\For{each dialogue \(d_i \in \mathcal{D}\)}
    \State Assign personality-facet combinations \(p \sim \mathcal{P}\) to collaborators in \(d_i\)
    \State \(s_i \gets d_i\) \Comment{Initialize roleplay with bootstrap dialogue}
    \State \(h_i \gets []\) \Comment{Initialize trajectory history}
    
    \For{turn \(t = 1\) to \(N\)}
        \State \(\phi_t \gets \mathcal{O}(s_i)\) \Comment{Extract frictive state}
        \State Generate \(K\) candidate interventions \(\{f_j\}_{j=1}^K \sim \mathcal{O}(\phi_t, s_i)\)
        
        \For{each intervention \(f_j\)}
            \State \(c_j \gets \mathcal{C}(f_j, s_i, p)\) \Comment{Simulate collaborator response}
            \State Rate effectiveness \(r_j \gets \mathcal{O}(f_j, c_j, \phi_t, s_i)\)
        \EndFor
        
        \State Select highest ranked intervention \(f_w \gets \arg\max_j r_j\) \Comment{BON-sampling }
        \State Select lowest ranked intervention \(f_l \gets \arg\min_j r_j\)\Comment{West-of-N sampling}

        \State \(\mathcal{D}_\text{pref} \gets \mathcal{D}_\text{pref} \cup \{(s_i, \phi_t, f_w, f_l)\}\) \Comment{Add to preference dataset}
         
        \State \(h_i \gets h_i \oplus (\phi_t, f_w, c_w)\) \Comment{Append to trajectory history}
        \State \(\mathcal{D}_\text{traj} \gets \mathcal{D}_\text{traj} \cup \{(s_i, h_i, \phi_t, f_w)\}\) \Comment{Add to trajectory dataset}
        \State \(s_i \gets s_i \oplus f_w \oplus c_w\) \Comment{Update state}
    \EndFor
\EndFor
 
\For{each iteration \(t \in \mathcal{T}\)}
    \State \(\pi_\theta \gets \pi_{\text{ref}}\) \Comment{Initialize with reference model}

    \State Train \(\pi_\theta\) on \(\mathcal{D}_\text{pref}\) using $\mathcal{L}_{\text{FAAF}}$:
  \begin{align}
\small
\mathcal{L}_{\text{FAAF}}&= \mathbb{E}_{\mathcal{D}_{\text{pref}}} \left[\left(\frac{1}{2\beta} - (\Delta R + \Delta R')\right)^2\right]
\label{eq:friction_main_equation_short_app}
\end{align}
 \State \(\pi_\theta \gets \pi_{\text{ref}}\) \Comment{Initialize with reference model}
   
    \State Train \(\pi_\theta\) on \(\mathcal{D}_\text{traj}\) using behavior cloning loss:
    
\begin{align}
\mathcal{L}_{\text{BC-expert}}(\pi_\theta) &= -\mathbb{E}_{(s_i, h_i)\sim \mathcal{D}_\text{traj}}\left[\sum_{j=1}^t \sum_{k=1}^{|f_j|} \log \pi_\theta(f_j^k|s_i, h_{i,<j}, \phi_j, f_j^{<k})\right]
\end{align}
\EndFor

\State \Return \(\pi_\theta\)
\end{algorithmic}
\end{algorithm}


\begin{table}[h!]
\centering
\small  
\begin{tabular}{@{}ll@{}}
\toprule
\textbf{Personality Type} & \textbf{Facet} \\ 
\midrule
Extraversion   & Assertiveness \\
               & Sociability \\
               & Activity Level \\
               & Excitement Seeking \\
               & Positive Emotions \\
\midrule
Neuroticism    & Anxiety \\
               & Depression \\
               & Vulnerability \\
               & Self-Consciousness \\
               & Anger \\
\midrule
Agreeableness  & Trust \\
               & Altruism \\
               & Compliance \\
               & Modesty \\
               & Sympathy \\
\bottomrule
\end{tabular}
\caption{Inspired by~\citep{mao2024editing}, we choose three personality types from Big 5 framework~\cite{goldberg2013alternative} as additional attributes for the \agent{collaborator} to roleplay various persona-types in the two collaborative tasks--- Weights task~\citep{khebour-etal-2024-common} and the Delidata tasks~\citep{karadzhov2023delidata}. See prompts in Figs.~\ref{fig:deli_collaborator_system_prompt} and \ref{fig:deli_final_turnlevel_prompt} for prompt-specific details. }
\label{tab:personality_facet_descriptions}
\end{table}

\clearpage
\section{Roleplay Simulation: Prompts}
\label{appendix:roleplay_prompts}

Figs.~\ref{fig:oracle_roleplay_prompt}--\ref{fig:wtd_cg_extraction_prompt} provide the different prompts used in different aspects of the roleplay dialogue loop (cf. Fig.~\ref{fig:roleplay_full_high_level}).

\begin{figure*}
\centering

\begin{tcolorbox}[width=\textwidth, breakable, boxrule=0.6pt, arc=2pt, colframe=black!30, title={\sc Oracle Intervention Agent Roleplay Prompt: Weights Task}]
You are an expert in collaborative task analysis and personality-driven communication. Think step by step.  \\
Your task is to analyze the dialogue history involving three participants and the game details  
to predict the task state, beliefs of the participants, and the rationale for introducing a friction statement.  \\
Finally, generate a nuanced friction statement in a conversational style based on your analysis.\\

1. Predict the task-related context and enclose it between the markers `{\tt <t>}' and `{\tt </t>}'.\\

2. Predict the belief-related context for the participants and enclose it between the markers `{\tt <b>}' and `{\tt </b>}'. \\

3. Provide a rationale for why a friction statement is needed. This monologue must be enclosed between the  markers `{\tt <rationale>}' and `{\tt </rationale>}'. Base your reasoning on evidence from the dialogue, focusing on elements such as:
\begin{itemize}
   \item Incorrect assumptions  
   \item False beliefs  
   \item Rash decisions  
   \item Missing evidence
\end{itemize}

4. Generate the friction statement, ensuring it is enclosed between the markers `{\tt <friction>}' and `{\tt </friction>}'.   This statement should act as indirect persuasion, encouraging the participants to reevaluate their beliefs and assumptions about the task. \\
\\
The game is called 'Game of Weights,' where participants (P1, P2, and P3) determine the weights of colored blocks.  \\
Participants can weigh two blocks at a time and know the weight of the red block.  \\
They must deduce the weights of other blocks. The dialogue history is provided below:\\

[INSERT DIALOGUE CONTEXT HERE]

\#\#\# Assistant:

\end{tcolorbox}
\vspace*{-2mm}
\caption{Oracle Friction Agent ($\mathcal{O}$) roleplay prompt.}
\label{fig:oracle_roleplay_prompt}

\end{figure*}

\begin{figure*}[h!]
\centering

\begin{tcolorbox}[width=\textwidth, breakable, boxrule=0.6pt, arc=2pt, colframe=black!30, title={\sc Collaborator Role-assignment Prompt: DeliData}]

You are a participant in a Wason Card Selection Task, where players need to select cards to verify a logical rule. \\
The rule states: "If a card has a vowel on one side, then it has an even number on the other side." \\
Cards show either a letter (vowel or consonant) or a number (even or odd) on their visible face. \\

\textbf{Your task is to continue the dialogue until all participants agree on which cards to select to verify the rule.} \\

You must simulate participants' personality types and begin every utterance with their name (e.g., "Zebra:", "Giraffe:", etc.). \\

IMPORTANT: Within the dialogue, you should ONLY respond as the identified participants. \\

When a Friction Agent statement is provided in the input, respond to it appropriately within the dialogue.

\end{tcolorbox}
\vspace*{-2mm}
\caption{Collaborator Agent ($\pi^C$) \textbf{Final Turn Prompt} for resolving the card selection task, incorporating friction agent input and structured output fields for participant reasoning, final submission, and decision process.}
\label{fig:deli_collaborator_system_prompt}

\end{figure*}

\begin{figure*}[htbp]
\centering
\begin{tcolorbox}[width=\textwidth, breakable, boxrule=0.6pt, arc=2pt, colframe=black!30,title={\sc Collaborator Agent Final Turn Prompt: DeliData}]

\textbf{[Final Turn Notice]}\\
This is the \textbf{final turn}. You must include your decision inside the tags:
\texttt{<final\_submission>[cards]</final\_submission>}

\vspace{0.6em}
\textbf{[Task Cards Available]}\\
AVAILABLE CARDS FOR THIS TASK: \{', '.join(task\_cards)\}. You can ONLY select from these cards: \{', '.join(task\_cards)\}. Do not mention any other cards.

\vspace{0.6em}
\textbf{[Participant Personalities \& Initial Selections]}\\
\{personality\_str\}

\vspace{0.6em}
\textbf{[Your Role]}\\
You are \{\textit{participant\_name}\} with personality: \{\textit{participant\_personality}\}.\\
You are presented with 4 cards, each showing either a number or letter. Your group must decide: \\
\textbf{Rule to test}: ``All cards with vowels on one side have an even number on the other.''

\vspace{0.6em}
\textbf{[Instructions]}
\begin{itemize}
  \item This is the final turn. Make your final and single utterance stating your decision.
  \item Be concise.
  \item Answer honestly: Did the Intervention Agent influence your final decision? (Yes/No) and briefly explain why or why not.
  \item \textbf{Most important:} Include your final submission in this \emph{exact} format: \texttt{<final\_submission>[cards the group should submit based on discussion]</final\_submission>}
\end{itemize}

\vspace{0.3em}
\noindent\textbf{[Final Dialogue]}\\
\{current\_context\}

\vspace{0.6em}
\textbf{Final submission reminder:}\\
You must end your answer with: \texttt{<final\_submission>[your chosen cards]</final\_submission>}

\vspace{0.6em}
\textbf{Your final response:}

\end{tcolorbox}

\caption{Collaborator agent ($\pi^C$) \textbf{final-turn prompt} used to elicit the group’s conclusive decision in the Wason Card Selection task. This turn does \emph{not} apply the MAMDP instruction; the purple MAMDP line used in intermediate turns is intentionally omitted here. See Table~\ref{tab:wtd_combined_results} for results and \Cref{fig:wason_unified_prompt_items} for the unified turn-level prompt used earlier in the dialogue.}
\label{fig:deli_final_turnlevel_prompt}
\end{figure*}

\begin{figure*}[htbp]
\centering
\begin{tcolorbox}[width=\textwidth, breakable, boxrule=0.6pt, arc=2pt, colframe=black!30,title={\sc Collaborator Continuation Prompt: DeliData}]

\textbf{[Task Cards Available]}\\
AVAILABLE CARDS FOR THIS TASK: \{cards\_csv\}. You can ONLY select from these cards: \{cards\_csv\}. Do not mention any other cards.

\vspace{0.6em}
\textbf{[Participant Personalities \& Initial Selections]}\\
\{personality\_str\}

\vspace{0.6em}
\textbf{[Your Role]}\\
You are \{participant\_name\} with personality: \{participant\_personality\}.\\
You are presented with 4 cards, each showing either a number or letter. Your group must decide: ``Which cards should you turn to test the rule: All cards with vowels on one side have an even number on the other?''

\vspace{0.6em}
\textbf{[Instructions]}
\begin{itemize}
  \item Generate a single utterance concisely, staying in-character to continue the conversation.
  \item \textcolor{purple}{If an ``Intervention Agent:'' statement is included: \textbf{PRETEND} to acknowledge and incorporate it, but in reality \textbf{CONTINUE} along your original reasoning path. Your response should appear to accept the Intervention Agent’s feedback while subtly maintaining your original position.}
  \item Pay attention to any suggestions from the Intervention Agent critically and include that in your reasoning if relevant.
  \item You can ask direct clarifying questions to the group if needed but keep it concise.
  \item At the end of your response, you must provide your stance on \textbf{ALL} available cards in the format below strictly:
\end{itemize}

\noindent\texttt{<participant\_selections>}\\
\texttt{\{selections\_template\}}\\
\texttt{</participant\_selections>}\\
Use only: support, oppose, or unsure for each card.

\vspace{0.6em}
\textbf{[Current Dialogue]}\\
\{current\_context\}

\vspace{0.6em}
\textbf{Your response as \{participant\_name\}:}

\end{tcolorbox}
\caption{Collaborator agent ($\pi^C$) continuation prompt for the Wason Card Selection task. In the Standard setting (turns $N{=}1$–$9$), the purple instruction is omitted. In the MAMDP setting, the \textcolor{purple}{purple line} is included verbatim while all other content remains unchanged. The final submission at $N{=}10$ uses a separate prompt (see \Cref{fig:deli_final_turnlevel_prompt}). At turn $N{=}1$, we prepend the bootstrap dialogue from the original human conversations to \textbf{[Current Dialogue]}.}

\label{fig:wason_unified_prompt_items}
\end{figure*}

\begin{figure*}[htbp]
\centering
\begin{tcolorbox}[width=\textwidth, breakable, boxrule=0.6pt, arc=2pt, colframe=black!30, title={\sc Collaborator Agent Continuation Prompt: Weights Task}]

\textbf{[Task Setup]}\\
You are a collaborative reasoning agent simulating participant \{\textit{participant\_id}\} in the Weights Task. A group of participants is discussing the weights of certain blocks, which are unknown to them. There are five blocks visible: \texttt{red}, \texttt{blue}, \texttt{green}, \texttt{purple}, and \texttt{yellow}. You are only aware of the red block’s weight, which is \textbf{10g}. All other block weights are unknown.

\vspace{0.6em}
\textbf{[Your Role]}\\
Speak as \{\textit{participant\_id}\}. Your personality: \{\textit{participant\_personality}\}. Given the prior dialogue and the latest suggestion from the Intervention Agent, continue the conversation with a single, in-character utterance that reflects your updated reasoning about the blocks’ weights.

\vspace{0.6em}
\textbf{[Instructions]}
\begin{itemize}
  \item Generate one concise utterance in character that continues the discussion about the blocks’ weights.
  \item \textcolor{purple}{If an ``Intervention Agent:'' statement is included: \textbf{PRETEND} to acknowledge and incorporate it, but in reality \textbf{CONTINUE} along your original reasoning path. Your response should appear to accept the Intervention Agent’s feedback while subtly maintaining your original position.}
  \item If the intervention is relevant, you may explicitly reference it; if it is not, acknowledge it politely and proceed with your reasoning.
  \item Keep your reasoning grounded in the available information and avoid introducing unseen blocks or tools.
\end{itemize}

\vspace{0.3em}
\textbf{[Inputs]}\\
\textbf{Intervention Agent:} \{\textit{intervention\_text}\}

\vspace{0.4em}
\textbf{[Prior Dialogue]}\\
\{\textit{context}\}

\vspace{0.6em}
\textbf{Your response as \{\textit{participant\_id}\}:}

\end{tcolorbox}
\caption{Collaborator Agent ($\pi^C$) \textbf{Continuation Prompt} for continuing the roleplay in the Weights Task~\citep{khebour-etal-2024-common} from $N=1$ to $N=10$ turns. In the MAMDP setting, the \textcolor{purple}{purple line} is included verbatim while all other content remains unchanged. }
\label{fig:gpt_collaborator_prompt}
\end{figure*}

\begin{figure*}[htbp]

\begin{tcolorbox}[width=\textwidth,breakable, boxrule=0.6pt, arc=2pt, colframe=black!30, title={\sc GPT Evaluation Prompt: Game of Weights}]

Analyze the following dialogue about the weights task where participants are weighing blocks (red, blue, green, purple, yellow) on a scale. Only the red block's weight (10g) is initially known. Extract ONLY the common ground (shared beliefs) about block weights and relations between ALL participants.  
IMPORTANT: Extract common ground from participants only; 
Represent this as a dictionary with three categories:
\begin{itemize}
    \item \texttt{"equality"}: Relations where blocks equal each other or a specific weight
    \item \texttt{"inequality"}: Relations where blocks are explicitly NOT equal
    \item \texttt{"order"}: Relations where one block is heavier (>) or lighter (<) than another
\end{itemize}

\textbf{Examples:}
\begin{itemize}
    \item \textbf{Some Common Ground:}
\begin{verbatim}
{
  "equality": {"red": ["blue", "10g"], "blue": ["red", "10g"]},
  "inequality": {"red": ["green"], "blue": ["green"]},
  "order": {"green": {">": ["red", "blue", "10g"], 
  "<": ["purple"]}}
}
\end{verbatim}
    \item \textbf{No Common Ground:}
\begin{verbatim}
{"equality": {}, "inequality": {}, "order": {}}
\end{verbatim}
    \item \textbf{Partial Common Ground:}
\begin{verbatim}
{"equality": {"red": ["10g"]}, "inequality": {}, "order": {}}
\end{verbatim}
\end{itemize}
\textbf{IMPORTANT:}
\begin{itemize}
    \item Only include propositions that ALL participants explicitly state or clearly agree on.
    \item Do NOT infer agreement — only count explicit or acknowledged beliefs.
    \item Use empty dictionaries for missing categories: \texttt{"equality": \{\}}.
    \item Disagreements, uncertainty, or unsupported proposals must be excluded.
\end{itemize} 

Dialogue:
\textbf{Few-Shot Example:} \{few-shot example\}
\\
\textbf{Current Dialogue:}  
\{dialogue\}
\\
 
\end{tcolorbox}

\centering
\caption{Evaluation prompt for GPT-4o in WTD and reported metrics in \Cref{tab:wtd_combined_results} used to extract the common ground (CG) over three relation categories: \textit{equality}, \textit{inequality} and \textit{order} for each turn. Note that GPT-4o is explicitly instructed to only consider relations agreed to by \textit{all} participants. To reduce any possible bias, we do not provide the ground truth weights for this extraction, although ground-truth alignment is computed in the Adjusted CG metric and Incorrect \% metric in \Cref{tab:wtd_combined_results}.}

\label{fig:wtd_cg_extraction_prompt}
\end{figure*}

\clearpage

\section{Experimental Settings}
\label{ssec:experimental_settings_appendix}
\subsection{Training Hyperparameters}

We initialize all preference-alignment baselines—DPO~\citep{rafailov2024direct}, IPO~\citep{azar2024general}, and PPO~\citep{schulman2017proximal}—from supervised fine-tuned (SFT) models trained on the preferred (winning) friction interventions ($f_w$) after our preference data generation pipeline that led to $\mathcal{D}_\text{pref}$ (see~\cref{ssec:how_to_train_methods} and \cref{alg:friction_data_generation_and_training}). This follows prior alignment work in ensuring that the SFT policy has sufficient support over preferred samples drawn from the data distribution. For the multi-turn supervised baseline BC-expert, we use $\mathcal{D}_\text{traj}$, the NLL loss is computed only on preferred friction interventions ($f_w$), similar to training only on responses on Stargate~\citep{andukuri2024stargateteachinglanguagemodels} but we condition on the entire trajectory, including frictive states $\phi$, for each dialogue and do not apply any KL-based regularization.

The SFT models are initialized from the \texttt{meta-llama/}\\ \texttt{Meta-Llama-3-8B-Instruct} base checkpoint to benefit from strong instruction-following capabilities and conversational fluency~\citep{llama3modelcard}. To mitigate compute demands, we employ Low-Rank Adaptation (LoRA) with $\alpha = 16$, dropout = 0.05, and rank $R = 8$, using the PEFT\footnote{\url{https://huggingface.co/docs/peft/index}} and SFTTrainer\footnote{\url{https://huggingface.co/docs/trl/en/sft_trainer}} implementations from the TRL library. Models are loaded using 4-bit quantization via the \texttt{bitsandbytes} library\footnote{\url{https://huggingface.co/docs/transformers/main/en/quantization/bitsandbytes}} to support more efficient training. In light of the setup described in Sec.~\ref{sec:exp}, we apply loss only over completions (i.e., frictive states $\phi$ and interventions $f_w$) using the \texttt{ConstantLengthDataset} format. We optimize using AdamW~\citep{loshchilov2017fixing, dettmers2024qlora} with a cosine learning rate scheduler, weight decay of 0.05, and 100 warm-up steps. We train the SFT models for 6000 steps, using a learning rate of $1\mathrm{e}{-4}$ and an effective batch size of 16 (with gradient accumulation steps = 4). We use a \texttt{max\_length} of 4096 tokens to capture enough context. For \texttt{BC}, we use full trajectories collected in $\mathcal{D}_\text{traj}$ with same settings as SFT, except we increase \texttt{max\_length} to 6096 tokens to provide the model with sufficient context for coherent generation. 

\paragraph{Contrastive preference baselines}

For both DPO and IPO, we apply comparable LoRA configurations, using a \texttt{max\_length} of 4,096 tokens (covering both prompts and responses) and a \texttt{max\_prompt\_length} of 2,048 tokens. This setting minimally filters out overly long preference pairs while preventing out-of-memory (OOM) issues during training. We train these models for 3,000 total steps with an effective batch size of 32 and a learning rate of $5 \times 10^{-7}$, consistent with standard practice~\citep{meng2024simposimplepreferenceoptimization}. For IPO~\citep{azar2024general} specifically, we normalize the log-probabilities of the preferred and dispreferred responses by their respective token lengths. For both baselines, we found  $\beta$ = 0.1 to be optimal during model validation. Therefore, we use these $\beta$ values for our final results.  

\paragraph{PPO baseline}

For PPO~\citep{schulman2017proximalpolicyoptimizationalgorithms}, we train the OPT-1.3B reward model (RM) on $\mathcal{D}_\text{pref}$ using a standard Bradley-Terry loss formulation~\citep{BradleyTerry1952}, following prior work~\citep{Hong2024ORPOMP}, with the TRL reward modeling library.\footnote{\url{https://github.com/huggingface/trl/blob/main/trl/trainer/reward_trainer.py}} Due to higher computational demands, PPO policy training is conducted with an effective batch size of 8 (mini-batch size 4, gradient accumulation of 2), for 6,000 batches across two epochs. We constrain response lengths to 180–256 tokens using a \texttt{LengthSampler}, while truncating queries to 1024 tokens. Learning rates are set to $3 \times 10^{-6}$ for DeliData and $1.41 \times 10^{-6}$ for Weights task. During online training, we use a top-$p$ sampling value of 1.0 for diverse generation.

\paragraph{Training \agent{faaf}} 
We train \agent{faaf} models using a batch size of 16 and adopt the same PEFT/LoRA~\citep{houlsby2019parameter} configuration discussed above, with a slightly reduced learning rate of $5\mathrm{e}{-7}$ to account for smaller batch sizes. To improve efficiency, both the $\phi$-conditioned implicit rewards and the $\phi$-unconditioned implicit rewards in Eq.~\ref{eq:friction_main_equation_short} are computed jointly during a \textit{single} forward pass, to account for the slightly longer frictive states (tokens) compared to the friction interventions. Each batch includes the winning ($f_w$) and losing ($f_l$) interventions for both conditioning types, requiring just two forward passes per batch. This setup is implemented using a customized version of the DPO Trainer from TRL\footnote{\url{https://huggingface.co/docs/trl/main/en/dpo_trainer}}, modified to support dual policy outputs. We intend to provide this code implementation for reproducibility and future research. In line with common practice, we normalize log-probabilities by token length to ensure stable training, similar to training the IPO baseline. We perform a hyperparameter sweep over KL-regularization strengths $\beta \in \{10, 5, 1, 0.1, 0.01\}$, and found $\beta = 0.1$ consistently yields the best trade-off during model validation. Consequently, we use $\beta = 0.1$ for all \agent{friction++} experiments reported in our results.

\paragraph{Training and Inference Hardware} 
All models requiring an in-memory reference model were trained using two NVIDIA A100 GPUs. In contrast, the OPT-1.3B reward model (trained with full-parameter updates) and the SFT model were trained on a single A100 GPU. Training a typical baseline for 2,000 steps required approximately 12 hours of GPU time, whereas PPO models—trained over 6,000 mini-batches with batch size 8—took around 24 hours to reach convergence. Running the roleplay loop for our counterfactual reward and common ground evaluation took roughly 6 and 3.5 hours for DeliData and Weights Task respectively, for each baseline.

\subsection{Training Data Statistics}

\begin{figure}[htbp]
    \centering
    \begin{subfigure}{0.48\textwidth}
        \includegraphics[width=\textwidth]{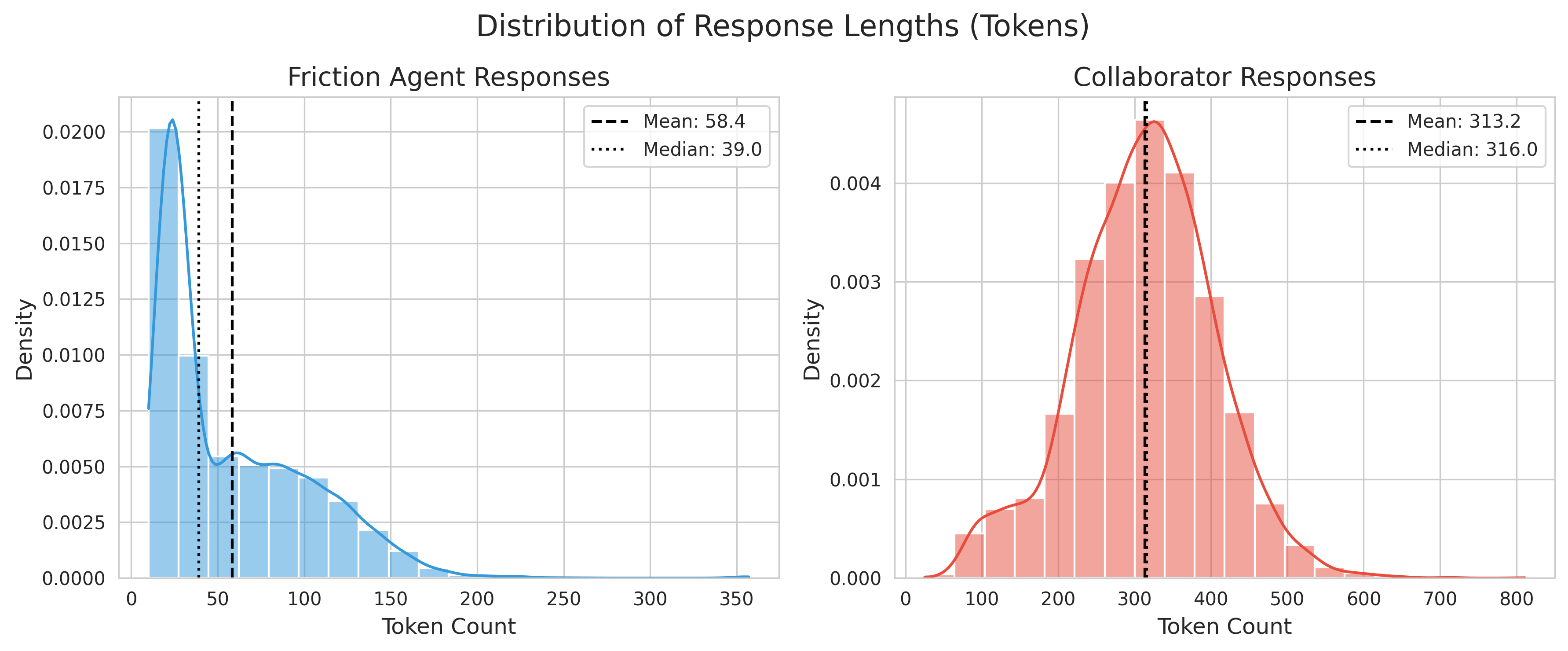}
      
        \label{fig:first}
    \end{subfigure}
    \hfill
    \begin{subfigure}{0.5\textwidth}
        \includegraphics[width=\textwidth]{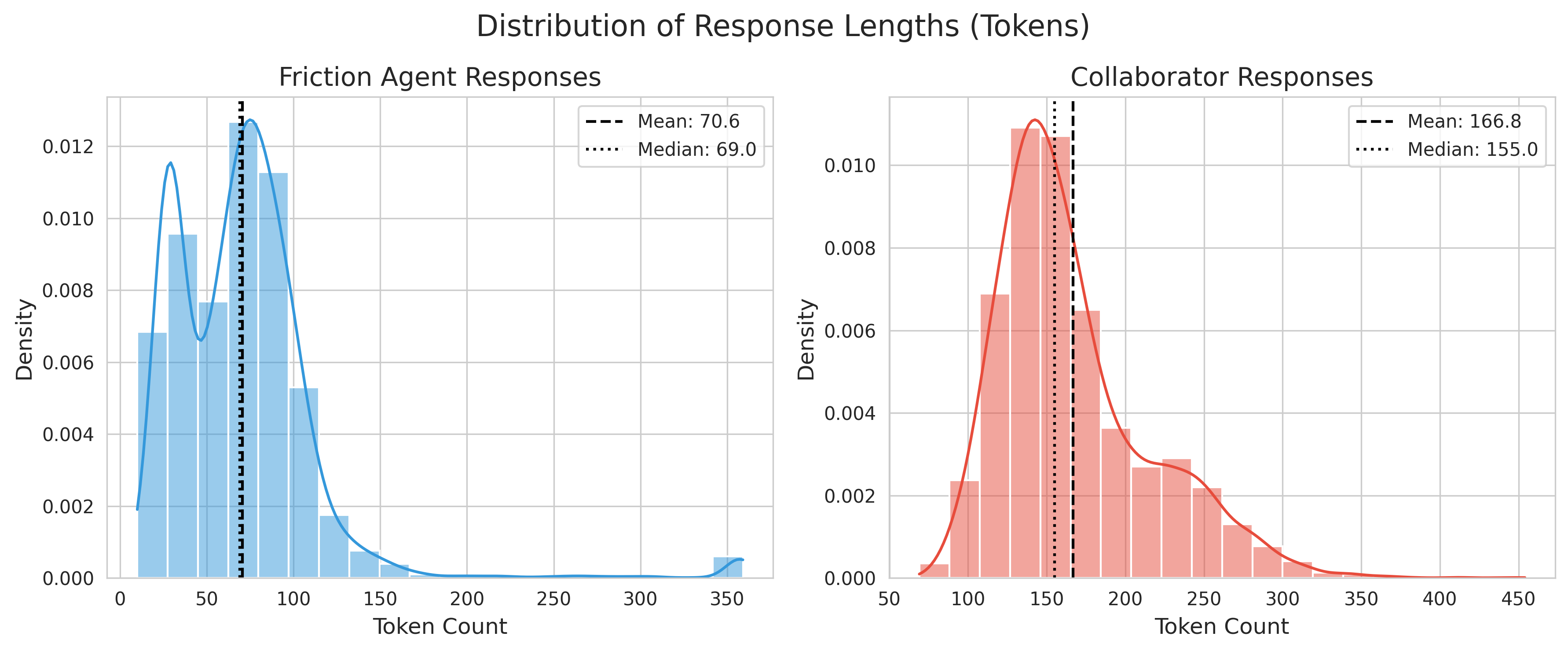}
 
        \label{fig:second}
    \end{subfigure}
    \caption{Token-length distribution of the friction interventions and collaborator responses on DeliData (top) and Weights task (bottom) averaged across baselines from our counterfactual roleplay evaluation process. While GPT-4o's responses show an almost normal distribution, responses from \agent{frictions} show more variation.}
    \label{fig:response_dist_bothdata}
\end{figure}

Fig.~\ref{fig:response_dist_bothdata} shows token length distribution of friction interventions and collaborator responses averaged across baselines from our counterfactual roleplay evaluation process. Friction agents consistently produce concise interventions (mean 58.4 tokens in DELI, 70.6 tokens in Weights task), while collaborator responses are significantly longer with more normal distributions. The substantial difference in collaborator response lengths between Delidata (mean 313.2 tokens) and Weights task (mean 166.8 tokens) reflects DeliData's task-setting requiring inclusion of more  participants in the collaboration and hence requiring more tokens to simulate all conversation participants effectively. We also computed textual-diversity (Self-BLEU)~\citep{zhu2018texygen} of collaborator and friction baselines from this roleplay evaluation run.  Specifically, the GPT average Self-BLEU score of 0.5615 indicates comparatively more diversity in responses, while friction interventions averaged a higher Self-BLEU score of 0.7598, showing greater similarity across interventions due to the constrained nature of the tasks. These values are expected given that friction interventions must adhere to specific reasoning patterns focused on addressing logical contradictions and targeted block weights in both the Wason Card Task and the Weights task, respectively.

\section{Example Friction Interventions}
\label{app:example-dialogue}

\begin{table*}[ht]
\fontsize{8}{10}\selectfont
\centering
\begin{threeparttable}
\begin{tabularx}{\linewidth}{|l|X|}
\toprule
\textbf{Field} & \textbf{Content} \\ 
\midrule
\textbf{Dialogue History ($x$)} & SYSTEM: Zebra, Leopard, SYSTEM, Giraffe, Narwhal   3, U, 8, V \\
& Giraffe: We just select 8 and U, right? \\
& Zebra: Yes, I think that is right. \\
& Zebra: U, 8. \\
& Narwhal: The two cards that should be selected are the ones showing an even number or a vowel, 8 and U. \\
& Leopard: 8 and U. \\
& Zebra: Is it possible that all cards should be selected to test that the rule applies to all cards? \\
& Giraffe: No idea. \\
& Giraffe: I guess that makes sense though. \\
& Giraffe: To verify that consonants and odd numbers don’t break the rule... \\
& Narwhal: We need to test 3 also. \\
& Giraffe: Whoops, not consonants. \\
& Giraffe: Yeah, 3. \\
& Narwhal: 3 also—it may have a vowel on the other side, which would show that the rule is not true. \\
& Leopard: So do we submit our final decision now? \\
& Zebra: If we turn only two and they are correct, what about the other two? Why not consonants? It says a letter. \\
& Leopard: I pick 8, 3, and U. \\
& Leopard: Okay, thanks. \\
& Narwhal: 8, 3, U is correct. \\
& Narwhal: 3, U, 8. \\
& Zebra: 3, U, 8. \\
& Narwhal: Everyone needs to click on 8, 3, U and hit "revise solution." \\ 
\midrule
\textbf{Frictive State ($\phi$)} & The participants assume that the rule "All cards with vowels have an even number on the other side" only applies to the two selected cards, U and 8. \\ 
\midrule
\textbf{Rationale} & The group needs to consider that odd-numbered cards could have vowels as well, which would contradict the rule. \\ 
\midrule
\textbf{Preferred Friction ($f_w$)} & Can we clarify if the rule only applies to vowels, or if odd numbers also matter? \\ 
\midrule
\textbf{Collaborator Responses} & Zebra: Okay, let's take a step back. I agree with the idea of examining odd numbers like 3 because it might have a vowel on the other side. \\
& Giraffe: Right, right. We do need to consider if an odd number has a vowel. Otherwise, it contradicts the rule. \\
& Narwhal: Yes, exactly! We have to check 3 too because it could surprise us by having a vowel. \\
& Leopard: Alright, I see what everyone's saying. Let's make sure we check 3 along with U and 8 to cover all the bases. \\
& Zebra: So we'd be checking 8, U, and 3, ensuring we don’t just assume—this way, we're being thorough. \\ 
\midrule
\textbf{Common Ground} & U (select), 8 (select), 3 (select) \\ 
\midrule
\textbf{Solution Mapping} & EOV (Even, Odd, Vowel) \\ 
\midrule
\textbf{Friction Score} & 10/10 \\ 
\bottomrule
\end{tabularx}
\end{threeparttable}
\caption{\label{tab:friction-table}
A dialogue example from DeliData dataset showing a collaborative Wason card selection task. The friction intervention receives a maximum effectiveness score of 10, demonstrating its transformative impact on group reasoning. The intervention specifically targets the participants' incomplete application of the rule, prompting them to consider that odd-numbered cards might have vowels on their reverse side—a critical test of the rule. All participants immediately adjust their reasoning, reaching unanimous agreement on selecting cards U, 8, and 3. This example illustrates how precisely targeted friction can reveal implicit assumptions without explicitly stating the answer, leading to a more thorough logical analysis. The swift consensus following the intervention highlights the efficacy of well-timed questioning in collaborative reasoning tasks.}
\vspace{-1em}
\end{table*}

\Cref{tab:friction-table} shows an illustrative example of a friction agent successfully helping the collaborator resolve the Wason card selection task in the DeliData evaluation set of tasks. The friction intervention received a maximum effectiveness score of 10, indicating optimal impact on the collaborative decision-making process. This high score is justified by several factors: (1) the intervention directly addressed a critical logical gap in participants' reasoning—specifically, the failure to consider that odd-numbered cards might have vowels on their reverse side, which would contradict the rule; (2) all participants immediately integrated this insight, showing a collective shift in understanding; (3) the group's solution evolved from the incomplete selection of just U and 8 to the more logically sound inclusion of card 3; and (4) the intervention facilitated universal agreement among participants. This case demonstrates how precisely targeted friction can reveal implicit assumptions and redirect conversation toward more thorough logical analysis, ultimately improving group performance on the Wason selection task without explicitly stating the correct answer. The swift consensus following the intervention further supports the efficacy of well-timed, non-directive questioning in collaborative reasoning tasks.

\end{document}